\documentclass[letterpaper]{article}
\pdfoutput=1
\usepackage{uai2018}  
\usepackage[margin=1in]{geometry}

\usepackage{math_commands}  
\newcommand{\bu}{\bm u}
\newcommand{\bv}{\bm v}
\newcommand{\strgconvex}{\mu}

\usepackage[utf8]{inputenc} 
\usepackage[T1]{fontenc}    
\usepackage[american]{babel}  

\usepackage[round]{natbib}
\usepackage{cite}

\usepackage{paralist} 

\usepackage{times}  
\usepackage{url}            
\usepackage{booktabs}       
\usepackage{microtype}      
\usepackage{algorithm}  
\usepackage[noend]{algorithmic}   
\usepackage{graphics}  
\usepackage{graphicx}
\usepackage{wrapfig}
\usepackage{subcaption}  
\usepackage{adjustbox}  
\usepackage{color}
\usepackage{multicol}
\definecolor{mydarkblue}{rgb}{0,0.08,0.45}
\usepackage{pgf}
\pgfkeys{/pgf/number format/.cd ,precision=1,sci generic={exponent={\times 10^{#1}}}}
\newcommand\convert[1]{\pgfmathprintnumber{#1}}

\newtheorem{theorem}{Theorem}
\newtheorem{lemma}[theorem]{Lemma}

\usepackage{hyperref}       
\hypersetup{ %
    pdftitle={Adaptive Stochastic Dual Coordinate Ascent for Conditional Random Fields},
    pdfkeywords={},
    pdfborder=0 0 0,
  pdfpagemode=UseNone,
    colorlinks=true,
    linkcolor=mydarkblue, 
    citecolor=mydarkblue, 
    filecolor=mydarkblue, 
    urlcolor=mydarkblue, 
    pdfview=FitH,
    pdfauthor={Rémi Le Priol, Alexandre Piché, Simon Lacoste-Julien},
}

\usepackage[all]{xy}  

\title{Adaptive Stochastic Dual Coordinate Ascent \\ for Conditional Random Fields}

\author{
    Rémi Le Priol \\
    MILA and DIRO \\
    Université de Montréal, Canada \\
    \And Alexandre Piché \\
    MILA and DIRO \\
    Université de Montréal, Canada \\
    \And Simon Lacoste-Julien \\
    MILA and DIRO \\
    Université de Montréal, Canada \\
}

\begin{document}

\maketitle

 \begin{abstract}
 \vspace{-2mm}
This work investigates the training of conditional random fields (CRFs) via the stochastic dual coordinate ascent (SDCA) algorithm of~\citet{shalev2016accelerated}.
SDCA enjoys a linear convergence rate and a strong empirical performance for binary classification problems.
However, it has never been used to train CRFs.
Yet it benefits from an ``exact'' line search with a single marginalization oracle call, unlike previous approaches.
In this paper, we adapt SDCA to train CRFs, and we enhance it with an adaptive non-uniform sampling strategy based on block duality gaps.
We perform experiments on four standard sequence prediction tasks.
SDCA demonstrates performances on par with the state of the art, and improves over it on three of the four datasets, which have in common the use of sparse features.
\end{abstract}

\vspace{-2mm}
\section{INTRODUCTION}

The conditional random field (CRF) model~\citep{lafferty2001conditional} is a common tool in natural language processing and computer vision for structured prediction.
The optimization of this model is notoriously challenging.
\citet{schmidt2015non} describes a practical implementation of the stochastic average gradient (SAG) algorithm~\citep{roux2012stochastic} for CRFs and proposes a non-uniform sampling scheme that boosts performance.
This algorithm (SAG-NUS) is currently the state of the art for CRFs optimization and we refer to \citet{schmidt2015non} for a detailed review of competing methods.

Deterministic (batch) methods such as L-BFGS~\citep{sha2003shallow,wallach2002efficient} have linear convergence rate but the cost per iteration is large.
On the other hand, the online exponentiated gradient method (OEG)~\citep{collins2008exponentiated} and SAG are both members of a family of algorithms with cheap stochastic updates and linear convergence rates, and they have both been applied to the training of CRFs.
They are called variance reduced algorithms, because their common point is to use memory to reduce the variance of the stochastic update direction as they get closer from the optimum.
\citet{johnson2013accelerating} coined the name stochastic variance reduced gradient (SVRG) and \citet{defazio2014saga} unified the family.

The stochastic dual coordinate ascent (SDCA) algorithm proposed by \citet{shalev-shwartz_stochastic_2013,shalev2016accelerated} is a member of this family that has not yet been applied to CRFs.
It is closely related to OEG in that it also does block-coordinate ascent on the dual objective.
Yet an interesting advantage of SDCA over OEG (and SAG) is that the form of its update makes it possible to perform an ``exact'' line search with only \emph{one} call to the \emph{marginalization oracle}, i.e. the computation of the marginal probabilities for the CRF.
This is in contrast to both SAG and OEG where each step size change requires a new call to the marginalization oracle.
We thus propose in this paper to investigate the performance of SDCA for training CRFs.

\textbf{Contributions.} We adapt the multiclass variant of SDCA to the CRF setting by considering the marginal probabilities over the cliques of the graphical model.
We provide a novel interpretation of SDCA as a relaxed fixed point update and highlights the block separability of the duality gap.
We propose to enhance SDCA with an adaptive non-uniform sampling strategy based on the block gaps, and analyze its theoretical convergence improvement over uniform sampling.
We compare the state-of-the-art methods on four prediction tasks with a sequence structure.
SDCA with uniform sampling performs comparably with OEG and SAG.
When SDCA is enhanced with the adaptive sampling strategy, it outperforms its competitors in terms of number of parameters updates on three of the tasks.
These three tasks are all about natural language with handcrafted sparse features.
We hypothesize that the efficiency of the dual methods can be related to the sparsity of these features.

\textbf{Related work.}
Our proposed gap sampling strategy is similar to the one from~\citet{osokin2016minding} in the context of SDCA applied to the structured SVM objective, which reduces to the block-coordinate Frank-Wolfe (BCFW) algorithm~\citep{lacoste2013block}.
\citet{dunner2017efficient} recently analyzed a general adaptive sampling scheme for approximate block coordinate ascent that generalizes SDCA.
Their proposed sampling scheme (which basically chooses the biggest gap) was motivated in the different context of mixed GPU and CPU computations, which does not apply to our setting.
Our proposed practical strategy takes in consideration the staleness of the gaps and is more robust in our experimental setting.
\citet{csiba2015stochastic} proposes an adaptive sampling scheme for SDCA for binary classification which unfortunately cannot be generalized to the CRF setting due to an intractable computation. Closely related to our work is \citet{perekrestenko17a} who analyzed several adaptive sampling strategies for a generalization of the primal-dual SDCA setup, including our proposed gap sampling scheme. However their analysis was focused on the single coordinate descent method (e.g. binary SDCA) and on sublinear convergence results obtained when strong convexity is not assumed. We cover instead the block-coordinate approach relevant to CRFs, and one of our notable results is to show that the linear convergence rate for gap sampling \textbf{dominates} the one for uniform sampling, in contrast to what happens in the sublinear regime studied by \citet{perekrestenko17a}. 

\textbf{Outline.}
We review the optimization problem for CRFs as well as provide novel insights on the primal-dual optimization structure in Section~\ref{sec:CRF}.
We present SDCA for CRFs in Section~\ref{sec:SDCA} and discuss important implementation aspects in Section~\ref{sec:implementation}.
We present and analyze various adaptive sampling schemes for SDCA in Section~\ref{Adaptive Sampling}.
We provide experiments in Section~\ref{sec:experiments} and discuss the implications in Section~\ref{sec:discussion}.

\section{CONDITIONAL RANDOM FIELDS} \label{sec:CRF}
In this section, we review the CRF model and its associated primal and dual optimization problems.
We then derive some interesting properties which motivate several optimization algorithms.

\subsection{DEFINITION}
A CRF models the conditional probability of a structured output $y \in \cY$ (e.g. a sequence) given an input $x\in \mathcal X$ with a Markov random field that uses an exponential family parameterization with sufficient statistics $F(x,y) \in \real^d$ and parameters $\bw \in \real^d$ : $p(y | x ; \bw) \propto \exp(\bw^{\top}F(x, y))$. The feature vector $F$ decomposes as a sum over the cliques $C \in \cC$ of the graphical model for $y$: $F(x, y) = \sum_C F_C(x, y_C)$, where $y_C$ denotes the subset of coordinates of $y$ selected by the indices from the set $C$. See Figure~\ref{crf example} for an illustration.

\begin{figure}
	\centering \includegraphics[width=.4\textwidth]{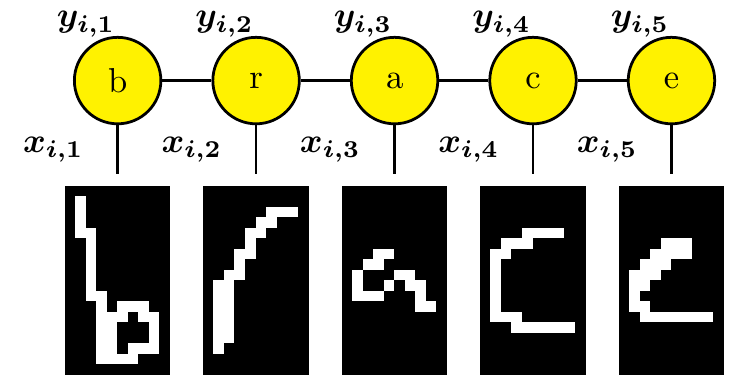}
	\caption{
		Example of graphical model for the optical character recognition (OCR) task.
		We want to exploit the structure of the word to predict that $y_{i,5}$ is an "e" and not a "c".
		This can be done by working on the pairs $y_{i,\{t, t+1\} } = (y_{i, t}, y_{i, t+1})$, the cliques of that model.
		}
		\label{crf example}
\end{figure}

\subsection{PRIMAL PROBLEM}
We have a data set $(x_i, y_i)_{i \in [1,n]}$ of $n$ i.i.d. input and structured output pairs.
The parameter is learned by minimizing the $\ell_2$-regularized negative log-likelihood:
\beq\label{negative log-likelihood}
	\min_{\bw \in \real^d} \frac{\lambda}{2}\| \bm w\|_2^2 + \frac{1}{n} \sum_{i=1}^{n} -\log\left(p(y_i | x_i ; \bw) \right) \, .
\eeq
We now rewrite it using the notation for the SDCA setup for multi-class classification from~\citet{shalev2016accelerated}.
Denote $M_i = |\cY_i|$ the number of labelings for sequence~$i$.
Denote $A_i$ the $d \times M_i$ matrix whose columns are the \textit{corrected features} $\{\psi_i(y) := F(x_i, y_i) - F(x_i, y)\}_{y \in \cY_i}$.
Denote also $\phi_i(s) := \log \big(\sum_{y \in \cY_i} \exp(s_y)\big)$ the log-partition function for the scores $s \in \real^{M_i}$. The negative log-likelihood can be written $-\log(p(y_i|x_i ;\bm w)) = \phi_i(-A_i^{\top} \bm w)$. The primal objective function to minimize over $\bm w \in \real^d$ thus  becomes:
\begin{equation}
	\label{eq:primal_problem}
	\cP(\bm w) := \frac{\lambda}{2}\| \bm w\|_2^2
	+ \frac{1}{n}   \sum_{i=1}^{n} \phi_i(-A_i^{\top} \bm w) \, .
\end{equation}

\subsection{DUAL FORMULATION}
The above minimization problem~\eqref{eq:primal_problem} has an equivalent {\it Fenchel convex dual} problem \citep{lebanon2002boosting}.
Denote $\Delta_{M}$ the probability simplex over $M$ elements.
Denote $\alpha_i \in \Delta_{M_i}$ the set of dual variables for a given $x_i$.
The dual problem handles directly the probability of the labels for the training set.
The dual objective to maximize over the choice of $\balpha = (\alpha_1, \ldots, \alpha_n) \in \Delta_{|\cY_1|} \times \ldots \times \Delta_{|\cY_n|} $ is:
\begin{equation}
	\label{dual problem}
	\cD(\balpha) := -\frac{\lambda}{2} \| \frac{1}{n \lambda} \sum_i A_i \alpha_i \|^2
	+ \frac{1}{n} \sum_{i=1}^n H(\alpha_i) \, ,
\end{equation}
where $H(\alpha_i) := - \sum_{y \in \cY_i} \alpha_i(y) \log(\alpha_i(y))$ is the entropy of the probability distribution $\alpha_i$. The negative entropy appears as the convex conjugate of the softmax: $-H = \phi^*$.

\subsection{OPTIMALITY CONDITION}
We define the \emph{conjugate weight} function $\hat{w}$ as follows:
\begin{multline*}
		\hat{w}(\balpha)
		:= \frac{1}{n \lambda} \sum_i A_i \alpha_i
		= \frac{1}{\lambda n} \sum_{i=1}^n \mathbb{E}_{y \sim \alpha_i} [\psi_i(y)] \\
		= \frac{1}{\lambda } \left( \frac{1}{ n} \sum_{i=1}^n F(x_i, y_i)
		-  \frac{1}{n} \sum_{i=1}^n \mathbb{E}_{y \sim \alpha_i} [F(x_i, y)] \right) \, .
\end{multline*}
It is the difference between the average of the ground truth features, and the average of the expected features for the dual variable, up to a factor $\frac{1}{\lambda}$.
We can show that $\hat{w}(\balpha^{\star}) = \bm w^{\star}$ where $\bm w^{\star}$ and $\balpha^{\star}$ are respectively the optimal primal parameters and the optimal dual parameters.

We can also  define the \emph{conjugate} probabilities $\hat{\alpha}_i$ as follows:
\begin{equation}
	\forall i, \quad \hat{\alpha}_i(\bm w) := \nabla_s\phi_i(-A_i^{\top} \bm w) = p(.|x_i; \bm w).
	\label{primal to dual}
\end{equation}
We get another optimality condition $\hat{\alpha}(\bm w^{\star}) = \balpha^{\star}$.
These two optimality conditions can be deduced directly from the structure of the duality gaps.

\subsection{DUALITY GAPS}\label{sec:duality gaps}
Note that $\cP(\bm w ) \geq \cD(\balpha)$ is always true, with equality at the optimum. The {\it duality gap} is defined by:
\beq
	g(\bm w, \balpha) = \cP(\bm w) -\cD(\balpha) \, .
\eeq
Note that we can rewrite the primal gradient as following:
\begin{equation}\label{primal gradient}
	\nabla \cP (\bw) = \lambda( \bw - \hat{w}\circ \hat{\alpha}(\bm w) ) \, .
\end{equation}
One can verify that:
\beqa
	\label{primal duality gap}
	g( \bm w,\hat{\alpha}(\bm w))
	& = & \frac{\lambda}{2} \|\bm w- \hat{w}(\hat{\alpha}(\bm w))\|^2 \\
	& = &  \frac{1}{2 \lambda} \| \nabla \mathcal P (\bm w) \|^2 \, . \label{eq:gradientGap}
\eeqa
This  structure of the gap for the primal weights and its conjugate dual probabilities have an equivalent in the dual.
Denote the Fenchel duality gap of $\phi_i$ for the scores $s_i = -A_i^T\bw$ and probabilities $\balpha_i$:
\begin{equation} \label{eq:Fench}
	F_i(s_i,\alpha_i) := \phi_i(s_i) + \phi_i^*(\alpha_i) + s_i^T \alpha_i \geq 0.
\end{equation}
The positivity comes from the definition of convex conjugates.
The gap is zero when $s_i$ and $\alpha_i$ are conjugate variables for $\phi_i$, e.g. $\alpha_i = \nabla \phi_i(s_i)$.
For any smooth loss~$\phi_i$, the duality gap between $\hat{w}(\balpha)$ and $\balpha$ decomposes as a sum of Fenchel gaps \citep{shalev-shwartz_accelerated_2013-1}:
\beqa
	\label{eq:Fench_blocks}
	g(\hat{w}(\balpha), \balpha)
	& = &\frac{1}{n} \sum_i F( -A_i^T \hat w (\balpha), \alpha_i).
\eeqa
The log-sum-exp and the entropy are a special pair of conjugates.
Their Fenchel duality gap is also equal to the Bregman divergence generated by $\phi_i^*=-H$, the Kullback-Leibler divergence: $F_i(s_i ,\alpha_i) = D_{KL}(\alpha_i || \nabla\phi_i(s_i) )$. Writing this for the same pair of conjugate variables yields:
\beqa
	\label{dual duality gaps}
	g(\hat{w}(\balpha), \balpha)
	& = &\frac{1}{n} \sum_i D_{KL} (\alpha_i || \hat{ \alpha}_i(\hat{w}( \balpha)).
\eeqa
The duality gaps~\eqref{primal duality gap}  and~\eqref{dual duality gaps} are typically used to monitor the optimization.
In Appendix~\ref{app:bound duality gap}, we explain how one can transfer a convergence guarantee on the primal or dual suboptimality to a convergence guarantee on the duality gap.\footnote{
    This implies that convergence results on the dual problem directly translates to convergence results on the primal and vice-versa;
     a fact apparently missed in the linear rate comparison of~\citet{schmidt2015non}.
}
Moreover, the block-separability of gaps from~\eqref{dual duality gaps} can motivate an adaptive sampling scheme, as we describe in Section~\ref{Adaptive Sampling}.

\subsection{INTERPRETATION}
The primal formulation chooses a $\bm{w}$ of small norm so as to maximise the conditional probability of observing the labels.
Conversely, the  dual formulation chooses conditional probabilities of the labels so as to minimize the $\ell_2$ distance between the expected features and empirical expectation of the ground truth features.
The optimal distribution would be the empirical  distribution, if not for the entropic regularization that favors more uniform probabilities.
This is the regularized version of the classical duality between maximum-likelihood and maximum-entropy for exponential families.

The optimality conditions show that the solution of the primal Problem~\eqref{eq:primal_problem} is also a \emph{fixed point} for the function~$\hat w \circ \hat \alpha$.
Because of the gradient form~\eqref{primal gradient}, the gradient descent update can also be written as a \emph{relaxed} fixed point update:
\begin{align}
		\bw^+
		& = \bw - \gamma \nabla \cP(\bw) \\
		& = (1-\gamma \lambda) \bw  + \gamma \lambda \,\, \hat w \circ \hat \alpha (\bw) \, .
\end{align}
The algorithm SDCA described in the next section also admits a relaxed fixed point update on the block $\alpha_i$ (see~\eqref{eq:dual_fixed_pt_update}).
More generally, optimization algorithms for Problem~\eqref{eq:primal_problem} can often be interpreted as a back and forth between the conjugate variables $w$ and $\hat w(\hat \alpha(\bm w))$ (primal methods) or $\alpha$ and $\hat \alpha(\hat w( \balpha))$ (dual methods).
For instance, one could interpret OEG as a relaxed fixed point iteration over the score variables $s_i = -A_i^T\bw$.
\begin{displaymath}
    \xymatrix{
    	\bm w \ar[r]^-{\hat{\alpha}}
    	&   \left( \nabla_s \phi_i(-A_i^T \bm w) \right)_{i=1}^n \ar[d] \\
		\frac{1}{n \lambda} \sum_i A_i \alpha_i  \ar[u]
		&  \balpha \ar[l]_-{\hat{w}}
	}
\end{displaymath}
Most of the results presented in this section and in Section~\ref{Adaptive Sampling} can be transposed to other kinds of loss and regularization, under some regularity assumptions.
Our focus in this paper is the application of SDCA to CRF models and thus we focused the discussion on the log-likelihood setting and the $\ell_2$ norm, which are widely used.

\section{PROXIMAL STOCHASTIC DUAL COORDINATE ASCENT} \label{sec:SDCA}

We first describe the SDCA in its general setting, and then describe the necessary modifications for training a CRF.

\subsection{GENERAL SETTING}
The stochastic dual coordinate ascent algorithm (SDCA) updates one dual coordinate at a time so as to maximize the dual objective.
SDCA was originally proposed for binary classification~\citep{shalev-shwartz_stochastic_2013} where each dual variable~$\alpha_i$ lives in $\Delta_2 = [0,1]$.
In this case, it is possible to do exact coordinate maximization of the dual objective over a single $\alpha_i$ with standard one dimensional optimization.

In the multi-class setting however, there is no simple way to maximize the dual objective over the block $\alpha_i \in \Delta_K$.
The algorithm with the surprising name of Proximal-SDCA\footnote{We simply call it SDCA in the rest of this paper}, option II~\citep{shalev2016accelerated} proposes a solution to this problem.
It updates $\alpha_i$  in a clever direction derived from the primal-dual relationship, which amounts to a relaxed fixed point update. See Algorithm~\ref{sdca general}.

\begin{algorithm}[t]
    \caption{Prox-SDCA (option II) called SDCA here}%
    \label{sdca general}
	\begin{algorithmic}
        \STATE Initialize $\alpha_i^{(0)} \in \Delta_{M_i}, \forall i$
        \STATE Let $\bw^{(0)} = \hat{w}(\bm{\alpha}^{(0)}) = \frac{1}{\lambda n} \sum_i A_i \alpha_i$
       \FOR{$t=0, 1\dots$}
                \STATE Sample $i$ uniformly at random in $\{1,\ldots,n\}$
                \STATE Let $ \beta_i := \hat{\alpha}_i(\bw) = \nabla_s \phi(-A_i^T \bw)$
                \STATE Let $\delta_i = \beta_i - \alpha_i^{(t)}$ \COMMENT{dual ascent direction}
                \STATE Let $\bv_i = \frac{1}{\lambda n} A_i \delta_i $ \COMMENT{primal direction}
                \STATE Solve Equation~\eqref{line search equation} to get $\gamma^*$ \COMMENT{Line Search}
               \STATE Update $\alpha_i^{(t+1)} := \alpha_i^{(t)} + \gamma^* \delta_i$
               \STATE Update $\bw^{(t+1)} := \hat{w}(\bm{\alpha}^{(t+1)}) = \bw^{(t)} + \gamma^* \bv_i $
        \ENDFOR
	\end{algorithmic}
\end{algorithm}

We now describe the idea.
At all time, we maintain the pair of dual and primal variables~$(\balpha, \bw = \hat w (\balpha))$.
At each step, we sample a training point~$i$.
We compute $\beta_i = \nabla_s \phi_i(-A_i^T \bw) = \hat \alpha_i \circ \hat w(\balpha)$,  the next fixed point iterate.
We then define the dual ascent direction by $\delta_i := \beta_i - \alpha_i$.
Finally we update the block~$\alpha_i$ with the right step size so as to increase the dual objective~$\cD(\balpha)$ using a relaxed fixed point update:
\begin{equation} \label{eq:dual_fixed_pt_update}
	\alpha_i^+ \leftarrow \alpha_i + \gamma \delta_i = (1-\gamma)\alpha_i + \gamma \hat \alpha_i \circ \hat w(\balpha) \, .
\end{equation}
The dual ascent direction is guaranteed to increase $\cD(\balpha)$, unless $\delta_i = 0$ (this actually means that the block is already optimal, see~\eqref{dual duality gaps}).
The primal weights $\bw = \hat w (\balpha)$ are related to $\balpha$ by a linear transformation.
Define the primal direction $\bv_i = \frac{1}{\lambda n} A_i \delta_i \in \real^d$.
One can update the weights directly: $\bw^+ \leftarrow \bw + \gamma \bv_i$.

The step size $\gamma \in [0,1]$  is either fixed, or found via line search.
In practice the fixed step size for which convergence is guaranteed is really small.
The line search is relatively cheap as we are looking at only one block:
\begin{equation}\label{line search equation}
	\gamma^*
	:= \argmax_{\gamma \in [0,1]} - \phi^*_i(\alpha_i + \gamma \delta_i)
	- \frac{\lambda n}{2} \| \bw + \gamma \bv_i \|^2.
\end{equation}
Note that one can decompose the quadratic term and precompute $ \langle \bw, \bv_i \rangle$ and $\| \bv_i \|^2 $ to accelerate the optimisation.
The bottleneck remains the computation of $\phi^*_i$ (and its derivatives).

\subsection{ADAPTATION TO CRF}
In the CRF setting, the dual variable~$\alpha_i$ is exponentially large in the input size~$x_i$.
For a sequence $x_i$ of length $T$ where each node can take up to $K$ values, the number of possible labels is $|\cY_i | = K^{T}$.
It might not even fit in memory.
Instead, the standard approach used in OEG and SAG is to consider the marginal probabilities $(\mu_C)_{C \in \cC}$ on the cliques of the graphical model.
Similarly, we replace $\balpha$ by $\bmu = (\mu_1, \cdots, \mu_n)$, where $\mu_i \in \prod_C \Delta_{C}$ is the concatenation of all the clique marginal vectors for the sample $i$.
For the same sequence $x_i$, this reduces the memory cost to $K^2(T-1)$ for the pair marginals.
We denote $m_i= \sum_C |\cY_{i, C} |$ this new memory fingerprint.
For a sequence long enough, we have $m_i \ll M_i$.
The associated weight vector can still be expressed as function of $\bmu$ thanks to the separability of the features:
\begin{equation}
	\label{weights from marginals}
	\hat w (\bmu) = \frac{1}{\lambda n} \sum_i \sum_C \mathbb E_{\mu_{i, C}}[\psi_{i, C}]
	= \frac{1}{\lambda n} \sum_i B_i \mu_i,
\end{equation}
where $B_i = (\psi_{i,C}(y_{C}) )_{C, y_C} \in \real^{d \times m_i}$ is the horizontal concatenation of the cliques feature vectors.

\begin{algorithm}[t]
    \caption{SDCA for CRF}%
    \label{sdca for crf}
	\begin{algorithmic}
        \STATE Initialize $\mu_{i}^{(0)} \in \prod_C \Delta_{C}$ consistently $\forall i$ \COMMENT{use~\eqref{eq:initialization}}
        \STATE Set $\bm{w}^{(0)}  :=  \hat{w}(\bm\mu^{(0)}) = \frac{1}{\lambda n} \sum_i B_i \mu_i^{(0)}$ \,\, \COMMENT{See \eqref{weights from marginals}}
        \STATE (Optional) Let $\quad g_i = 100, \forall i $
       \FOR{$t=0, 1\dots$}
                \STATE Sample $i$ uniformly at random in $\{1,\ldots,n\}$
                \STATE (Alternatively)  Sample $i$ proportionally to $g_i$
                \STATE Let $\nu_{i, C} (y_C) := p(y_C | x_i; \bw^{(t)}), \forall C\in\cC$  \COMMENT{\textbf{oracle}}
                \STATE (Optional) Let $g_i = \tilde D(\mu_i || \nu_i)$  \COMMENT{duality gap \eqref{divergence marginal}}
                \STATE Let $\delta_i = \nu_i - \mu_i^{(t)}$ \COMMENT{ascent direction}
                \STATE Let $\bv_i = \frac{1}{\lambda n} \hat w(\delta_i)$ \COMMENT{primal direction}
                \STATE Solve Equation~\eqref{crf line search} to get $\gamma^*$  \COMMENT{Line Search}
               \STATE Update $\mu_i^{(t+1)} := \mu_i^{(t)} + \gamma^* \delta_i$
               \STATE Update $\bm{w}^{(t+1)} := \hat{w}(\bm\mu^{(t+1)}) = \bm{w}^{(t)} + \gamma^* \bv_i $
        \ENDFOR
	\end{algorithmic}
\end{algorithm}

Now, assume that the graph has a {\it junction tree} structure $T=(\cC,\cS)$~\citep[Def.~10.3]{koller2009PGM}, where $\cC$ is the set of maximal cliques and $\cS$ the set of separators.
We can then run message passing on the junction tree to infer the new marginals given weights $\bw$: $\hat \mu_i(\bm w) = {p(y_C=. | x_i ;\bw)}$.
We can also now recover the joint probability $\alpha_i(y)$ as a function of its marginals $\mu_{i, C}$ \citep[Def.~10.6]{koller2009PGM}:
\begin{equation}
	\label{joint from marginals}
	\alpha_i(y) = \frac{\prod_{C\in\cC} \mu_{i, C} (y_C)}{\prod_{S \in\cS} \mu_{i, S}(y_S)}.
\end{equation}

Equation~\eqref{joint from marginals} in turn allows us to compute the entropy and the divergences of the joints, using only the marginals. Let $\mu_i$ and $\nu_i$ be the marginals of respectively $\alpha_i$ and $\beta_i$, then the entropy and the Kullback-Leibler divergence are given by:
\begin{equation}
	\label{entropy marginal}
	\tilde H (\mu_i)
	:= H (\alpha_i)
	= \sum_C H(\mu_{i, C}) - \sum_S H(\mu_{i, S})
\end{equation}
and
\begin{multline}
	\label{divergence marginal}
 	\tilde D (\mu_i || \nu_i)
 	:=D_{KL}(\alpha_i|| \beta_i) \\
 	= \sum_C D_{KL}(\mu_{i, C}||\nu_{i,C}) - \sum_S D_{KL}(\mu_{i, S}||\nu_{i, S}). 
\end{multline}

With this expression of the entropy~\eqref{entropy marginal}, we can compute the dual objective, and thus perform the line search:
\begin{equation}\label{crf line search}
	\gamma^* = \argmax_{\gamma \in [0,1]} \tilde H(\mu_i^{(t)} + \gamma \delta_i) - \frac{\lambda n}{2} \| \bw^{(t)} + \gamma \bv_i \|^2.
\end{equation}
With the Kullback-Leibler divergence~\eqref{divergence marginal}, we can compute efficiently the individual duality gaps from~\eqref{dual duality gaps}.
Algorithm~\ref{sdca for crf} describes this variation of SDCA, with as an option a non-uniform sampling strategy defined in  Section~\ref{ssec:gap_sampling}.

\section{IMPLEMENTATION} \label{sec:implementation}
We provide in Appendix~\ref{app:sec:implementation} a discussion of various important implementation aspects summarized here.
\begin{compactenum}
\item The initialization of dual methods for CRFs can significantly influence their performance. As explained in Appendix~\ref{app:sec:implementation}, we use: 
\begin{equation} \label{eq:initialization}
	\balpha^{(0)} := \varepsilon \bu + (1-\varepsilon) \bm\delta \, ,
\end{equation}
where $\bu$ is the uniform distribution on each block, $\bm\delta$ is a unit mass on each ground truth label and $\varepsilon$ is a small number.
\item Storing the dual variable may be expensive and one should allocate a decent amount of memory.
\item The line search requires computing the entropy of the marginals.
This is costly and we used Newton-Raphson algorithm to minimize the number of iterations.
This in turn requires storing the logarithm of the dual variable.
\end{compactenum}
\section{ADAPTIVE SAMPLING FOR SDCA} \label{Adaptive Sampling}

Recently, there has been a lot of attention on non-uniform sampling for stochastic methods.
The general goal is to sample more often points which are harder to classify and can bring more progress on the objective.
These methods are said to be \textit{adaptive} when the sampling probability changes during the optimization.
SDCA itself has had several adaptive schemes proposed.
In the following, we attempt to explain and relate these methods, and suggest new schemes that work well on our problem.

\subsection{ASCENT LEMMA}\label{ascent lemma}
We start by restating the ascent lemma from Equation~(25) in \citet{shalev-shwartz_accelerated_2013-1}.
This lemma inspires and supports all the strategies.

\paragraph{Ascent after sampling $i$:}
At iteration $t$, if we sample $i$ and take a step of size  $\gamma_i \in [0,1]$, we can lower bound the resulting dual improvement:
\begin{align}
	\label{one point descent}
    \hspace{-3mm}
    & n (\cD(\balpha^+) - \cD(\balpha)) \notag \\
    & \geq \gamma_i \underbrace{ \big [ \phi(-A_i^T \bw) + \phi^*(\alpha_i) + \bw^T A_i \alpha_i \big ] }
    _{ \textrm{Fenchel gap} =: g_i} \notag \\
    & \quad + \gamma_i \bigg ( \frac{(1-\gamma_i)}{2} - \frac{\gamma_i R_i}{2 \lambda n} \bigg )\| \beta_i - \alpha_i \|^2_1
\end{align}
where  $R_i := \|A_i\|^2_{1\rightarrow 2}  = \max_{y\in\cY_i} \| \psi_i(y) \|_2^2 $ is the squared radius of the corrected features for sample $i$.

Note that compared to the original text, we used the fact that the regularizer is the $\ell_2$ norm and the loss is $1$-smooth with respect to the $\ell_\infty$ norm.
We define $R:=\max_i R_i$, ${\bar R := \frac{1}{n} \sum_i R_i}$ and $\bar g := \frac{1}{n} \sum_i g_i$ the true duality gap (see~\eqref{eq:Fench}-\eqref{eq:Fench_blocks}).
We also introduce $L_i := \lambda + \frac{R_i}{n}$ an upper bound on the smoothness of loss $i$ plus regularizer for the $\ell_2$ norm.
We recall from Section~\ref{sec:duality gaps} that $g_i = D_{KL}(\alpha_i || \beta_i)$~\eqref{dual duality gaps}.
We give the name \textit{residual} to $d_i := \| \beta_i - \alpha_i \|^2_1$.

This lemma is derived with standard assumptions and inequalities on the smoothness of the loss and the strong convexity of the regularizer.
The first term of the lower bound is the ascent guarantee while the other term gives condition on the step-size to ensure progress.
We refer the reader to the original paper for more details.

To get the expected progress (conditioned on the past) after sampling with probability $\bm p$, we simply need to take the sum of the inequality above after multiplying both sides by $p_i$.
Our goal is to maximize this lower bound by choosing the right probability $\bm p$ and step sizes $\bm \gamma$.
To be able to conclude the proof with the original method, we also want some constants time the duality gap $\bar g$ to appear in the lower bound -- the gap is lower bounded by the dual suboptimality and thus this constant will give the linear rate of convergence.
The lemma can then transpose this result from the dual sub-optimality to the duality gap as described in Appendix~\ref{app:bound duality gap}.
From there on there are two general approaches: importance sampling and duality gap sampling.

\subsection{IMPORTANCE AND RESIDUAL SAMPLING}
With the importance sampling approach, the goal is to set the step-size and the probability so that they cancel each other out: $\gamma_i = \frac{\gamma}{p_i}$.
One then get an unbiased estimate of the true duality gap from~\eqref{dual duality gaps} as the first  term of the upper bound.
What is left is maximizing the second term with respect to $\bm p$.
This is the approach proposed by \citet{Zhao2015StochasticOptimizationImportance} (Importance Sampling, left term below) and generalized by \citet{csiba2015stochastic} (Residual sampling, a.k.a. AdaSDCA for binary classification, right term):
\begin{equation}
		p_i \propto L_i \quad \text{or} \quad p_i \propto d_i \sqrt{L_i}.
\end{equation}
These sampling schemes somehow allow to maximize the second term of~\eqref{one point descent}.
Intuitively, they replace a dependency on $R$ in the convergence rate by a dependency on $\bar R$.
They can give good results on binary and multi-class logistic regression. There are a few issues though.
\vspace{-\topsep}
\begin{itemize}
    \setlength{\parskip}{0pt}
    \setlength{\itemsep}{3pt plus 1pt}
    \item One needs an accurate estimate of the $L_i$.
    \item Importance sampling is not adaptive.
    \item In the CRF setting, the residual is $d_i = \| \beta_i -\alpha_i \|_1^2$.
    It is the squared $\ell^1$ norm of a vector of exponential size.
    We are not aware of any trick to compute it efficiently.
\end{itemize}
\vspace{-\topsep}

\subsection{GAP SAMPLING}\label{ssec:gap_sampling}
To make sure that the second term is positive, the original proof of uniform SDCA sets  $\gamma_i = \gamma = {(1+ \frac{R}{\lambda n})^{-1}}$ to obtain:
\begin{equation}
		n \mathbb E_p[D(\balpha^+) - D(\balpha)]
		\geq \gamma \sum_i p_i g_i.
\end{equation}
Assuming a full knowledge of the duality gaps $g_i$, the optimal decision is to sample the point with maximum duality gap.
This was done by \citet{dunner2017efficient} in the context of multi-class classification on a pair CPU-GPU. While the GPU computes the update, the CPU updates as many duality gaps as possible.
This lead to impressive acceleration over massive datasets.

However, this is not our current setting.
We know and update only one gap at a time (for efficiency).
Because of staleness of the gaps, our experiments with this method did not even converge for the most part (see Section~\ref{experiment sampling}).
We need a more robust method.

We take inspiration from what was done by \citet{osokin2016minding} to improve the Block-Coordinate Frank-Wolfe (BCFW) algorithm \citep{lacoste2013block}.
We propose to bias sampling towards examples whose duality gaps are large: $p_i \propto g_i$.
If we know all the duality gaps, the expected improvement reads:
\begin{equation}
	n \mathbb E_p[D(\balpha^+) - D(\balpha)]
		\geq \chi(\bm g)^2 \, \gamma \, \bar g,
\end{equation}
where $	\chi(\bm g) = \sqrt{ \frac{ \frac{1}{n} \sum_i g_i^2}{\bar g^2} } \in [1, \sqrt n] $ is the non-uniformity of the duality gaps, as defined in \citet[Section 3.1]{osokin2016minding}.
The value $\chi(\bm g)^2 \gamma$ is the value that will appear in the linear convergence rate of this method.
It means that the convergence rate for gap sampling \textbf{dominates} the one for uniform sampling.
This is different from what was observed for  BCFW where they could not prove dominance in general.

In practice we use stale estimates of the gaps and there are no convergence guarantees.
We discuss more this issue in section \ref{experiment sampling}.

We also explored a combination of gap sampling and importance sampling.
We could get similar convergence rate where a trade-off appeared between the mean smoothness and the non-uniformity.
We detail these considerations as a technical report in Appendix \ref{app:nusampling} for the interested reader.

\section{EXPERIMENTS} \label{sec:experiments}
We conducted these experiments to answer three questions:
(1) How does the line search influence SDCA?
(2) How do the non-uniform sampling schemes compare with each other?
and (3) How does SDCA compare with SAG and OEG on sequence prediction?

\subsection{EXPERIMENTAL SETTING}

We applied the experimental setup outlined by \citet{schmidt2015non}.
We implemented SDCA to train a classifier on four CRF training tasks: (1) the optical character recognition (OCR) dataset~\citep{taskar2004max}, (2) the CoNLL-2000 shallow parse chunking dataset (CONLL), (3) the CoNLL-2002 Dutch named-entity recognition dataset (NER), and (4) a part-of-speech (POS) tagging task using the Penn Treebank Wall Street Journal data.
Additional details regarding these datasets are provided in Table~\ref{datasets summary}.
Note that the tasks (2), (3), (4) are about language understanding.
They use sparse features (the ratio $a/A$ from the table is small).
The sparsest data set is NER.
Note that POS is considerably larger than other datasets.
All experiments are performed with a regularization factor $\lambda=1/n$.
We used our own implementation\footnote{The code to reproduce our experiments is available at: \url{https://remilepriol.github.io/research/sdca4crf.html}.} of SDCA coded in plain Python and Numpy \citep{walt2011numpy}.
In most plots we report the logarithm base 10 of the primal sub-optimality.
We got the optimum by running L-BFGS a large number of iterations.

\begin{table}[t]
	\caption{Dataset summary.
	$d$ is the dimension of $\bw$.
	$n$ is the number of data points (sequences).
	$N$ is the number of nodes (e.g. sum of sequences length).
	$K$ is the number of possible labels for each node.
	$A$ is the number of attributes (see Appendix~\ref{app:feature}).
	$a$ is the maximum number of attributes extracted from one node.
	Mem. is the memory required by the pairwise marginals stored as float 64.
	The pairwise marginals dominate the memory cost.}
	\label{datasets summary}
	{\small
		\begin{tabular}{lllll}
			\toprule
			Dataset          &   OCR  &   CONLL   &    NER   &      POS \\
			\midrule
			$d$ &   \convert{4082} &   \convert{1643026} &  \convert{2798955} &  \convert{8572770} \\
			$n$ &   \convert{6202} &      \convert{8936} &    \convert{15806} &    \convert{38219} \\
			$N$ &  \convert{52827} &    \convert{211727} &   \convert{202931} &   \convert{912273} \\
			$K$ &     26 &        22 &        9 &       45 \\
			$A$ &    128 &     \convert{74658} &   \convert{310983} &   \convert{190458} \\
			$a$ &    128 &        19 &       20 &       13 \\
			Mem.(GiB)      &    0.2 &       0.7 &      0.1 &       13 \\
			\bottomrule
		\end{tabular}
	}
	\end{table}

\subsection{EFFECT OF THE LINE SEARCH}\label{experiment line search}

We implemented the safe bounded Newton-Raphson method from \citet[Section 9.4]{press_numerical_1992} on the derivative of the line search function.
A natural question to ask is : how precise should the line search be?
The stopping criterion for this algorithm is the size of the last step taken so there is no proper precision parameter.
We refer to this stopping criterion for the line search as the sub-precision of SDCA.

We discovered experimentally that the convergence of SDCA is mostly  independent of the sub-precision.
On all datasets, if we ask 0.01 sub-precision or less, SDCA converges with the same rate.
An explanation is that the accuracy of the optimization arises from iterates $\balpha$ and $\hat \alpha(\hat w(\balpha))$ getting closer to each other in the simplex with each iteration.

Reaching 0.01 or 0.001 takes on average 2 iterations.
Each iteration of Newton's method require the computation of the first and second derivative of the line search objective~\eqref{crf line search}.
In the following we report results with sub-precision 0.001 to be on the safe side.
These 2 iterations were taking about 30\% of the algorithms running time for each dataset.\footnote{
We also tried initializing the line search with 0.5 or with the previous step size.
There was no significant difference.
}

We also performed experiments with only one step of the Newton update.
The convergence was not affected on OCR, CONLL and POS, but convergence failed on NER (see Figure~\ref{fig:subprecision} of Appendix~\ref{app:comp_plots}).
This phenomenon could be related to sparsity.

\subsection{COMPARISON OF SAMPLING SCHEMES} \label{experiment sampling}
We compare the performance of four sampling strategies with 20\% of uniform sampling against the full Uniform approach, on the OCR dataset (see results in Figure~\ref{fig:comparison sampling schemes}):\vspace{-2mm}
\begin{itemize}
  \item \textit{Importance:} sample proportionally to the smoothness constants $L_i = \lambda +\frac{R_i}{n}$. We report how we evaluated the radii $R_i$ in Appendix \ref{app:radius}. \vspace{-2mm}
  \item \textit{Gap:} sample proportionally to our current estimate of the duality gaps.\footnote{
  For the gap approaches, we initialize the gap estimates with large values (100) so as to perform a pass over the whole dataset before starting to sample proportionally to the stale estimates.
  } \vspace{-2mm}
  \item \textit{Gap $\times$ importance:} sample proportionally to the product of the gap and smoothness constants. \vspace{-2mm}
  \item \textit{Max:} sample deterministically the variable with the largest recorded gap \citep{dunner2017efficient}.
\end{itemize}
\begin{figure}[t]
\centering \includegraphics[width=.33\textwidth]{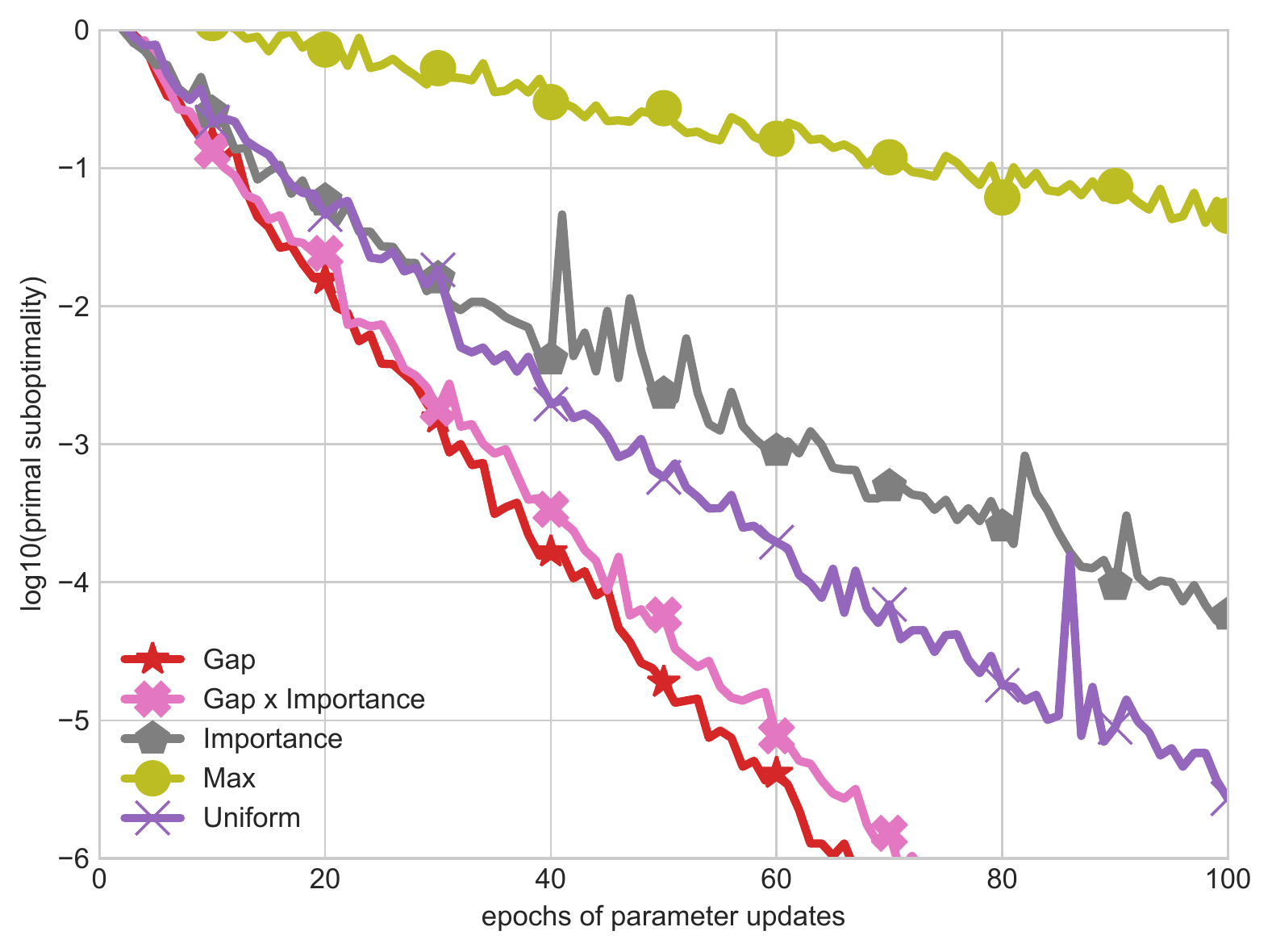}
\caption{Performance of competing sampling schemes on the OCR dataset with 80\% of non-uniformity. Sampling proportionally to the gap gives the best performance.}
\label{fig:comparison sampling schemes}
\end{figure}

As discussed in Section~\ref{ssec:gap_sampling}, Max sampling is not robust enough to the staleness of the gap estimates and fails to converge here.
We also observe that Importance performs worse than Uniform, and that Gap $\times$ Importance performs worse than Gap.
This  indicates that the smoothness upper bounds we estimated are not informative of the  difficulty of optimizing a point for SDCA.
Overall, Gap sampling  gives the best performance and this is what we use in the following experiments.

The ratio of uniform sampling is here to mitigate the fact that we sample proportionally to stale gaps.
This is the strategy adopted by SAG-NUS~\citep{schmidt2015non} which samples uniformly half of the time.
Another strategy used by~\citet{osokin2016minding} is to update all the duality gaps at once every 10 epochs or so.
Our experiments indicate that these strategies are not needed for SDCA-GAP.
Increasing the ratio of non-uniformity up to 1 only improves the performance on all datasets, though after 0.8 the improvements are marginal, as illustrated by Figure~\ref{fig:non-uniformity} for the NER dataset.
\begin{figure}[t]
\centering \includegraphics[width=.33\textwidth]{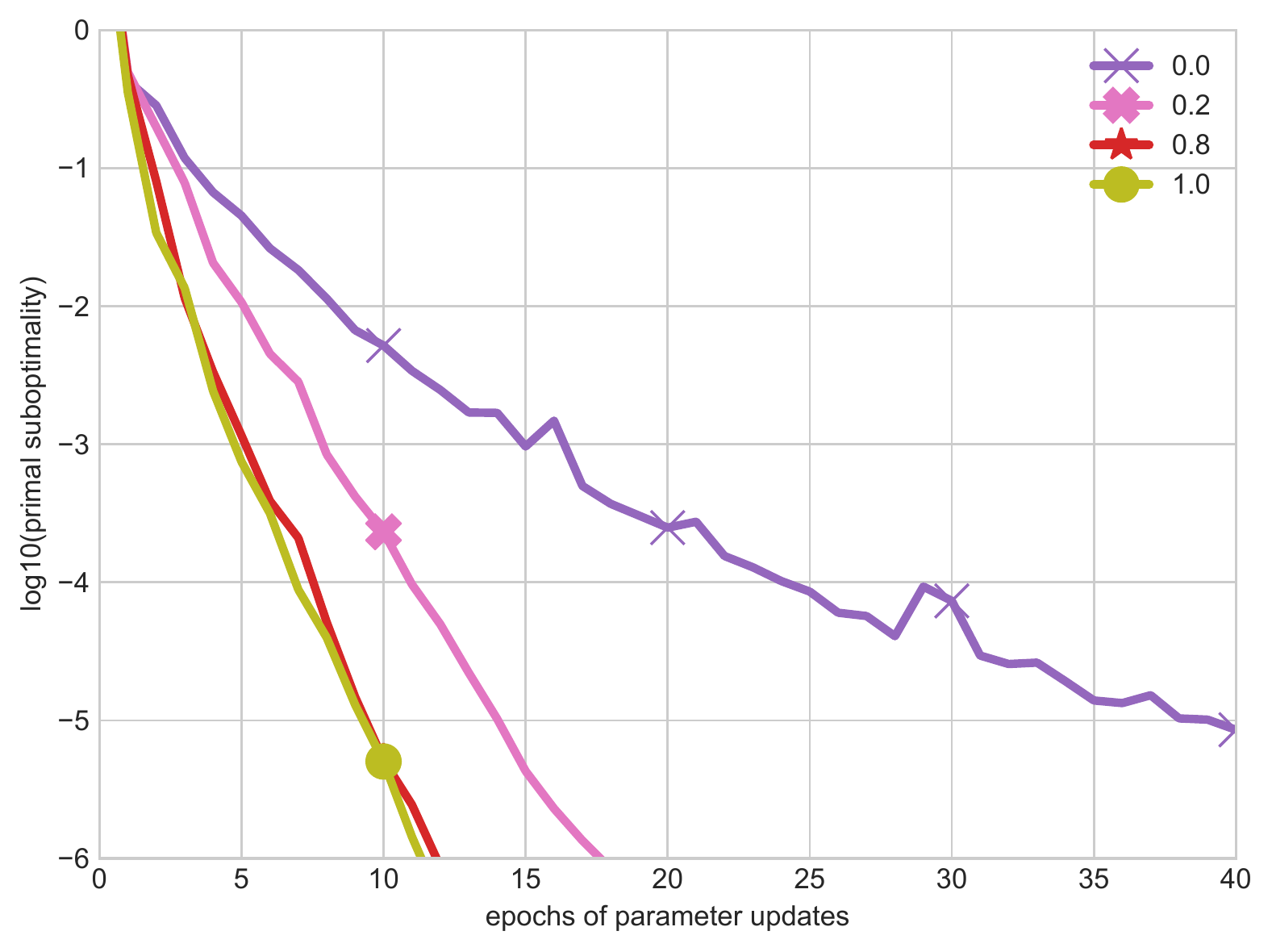}
\caption{SDCA with Gap sampling applied on NER with various fractions of non-uniform sampling, as indicated by the number in the legend.
Increasing the fraction only improves the performance, up to a certain point.}
\label{fig:non-uniformity}
\end{figure}

In fact, the estimate of the total gap maintained by SDCA is somewhat accurate, as illustrated for different datasets in Figure~\ref{fig:ratio} of Appendix~\ref{app:comp_plots}.
Empirically, it always remains within a factor 2 of the true duality gap.
This accuracy is a good news because one can use this estimate of the duality gap as a stopping criterion for the whole algorithm.
Once it reaches a certain precision threshold, one just has to perform one last batch update to check the real value.
This is similar in spirit to SAG, which uses the norm of its estimate of the true gradient as a stopping criterion.
Both are duality gaps estimators (see Equation~\eqref{primal duality gap}).

\subsection{COMPARISON AGAINST SAG AND OEG}

We downloaded the code for OEG and SAG-NUS as implemented by~\citet{schmidt2015non} from the SAG4CRF project page.\footnote{\url{https://www.cs.ubc.ca/~schmidtm/Software/SAG4CRF.html}}
We used our own implementation of SDCA with a line search sub-precision of $0.001$.
We provide the comparison in Figure~\ref{fig:comparison with SAG and OEG} according to two different measures of complexity which are implementation independent.

\begin{figure*}[t]
\centering
\begin{subfigure}{0.33\linewidth}
\centering
\includegraphics[width=\linewidth]{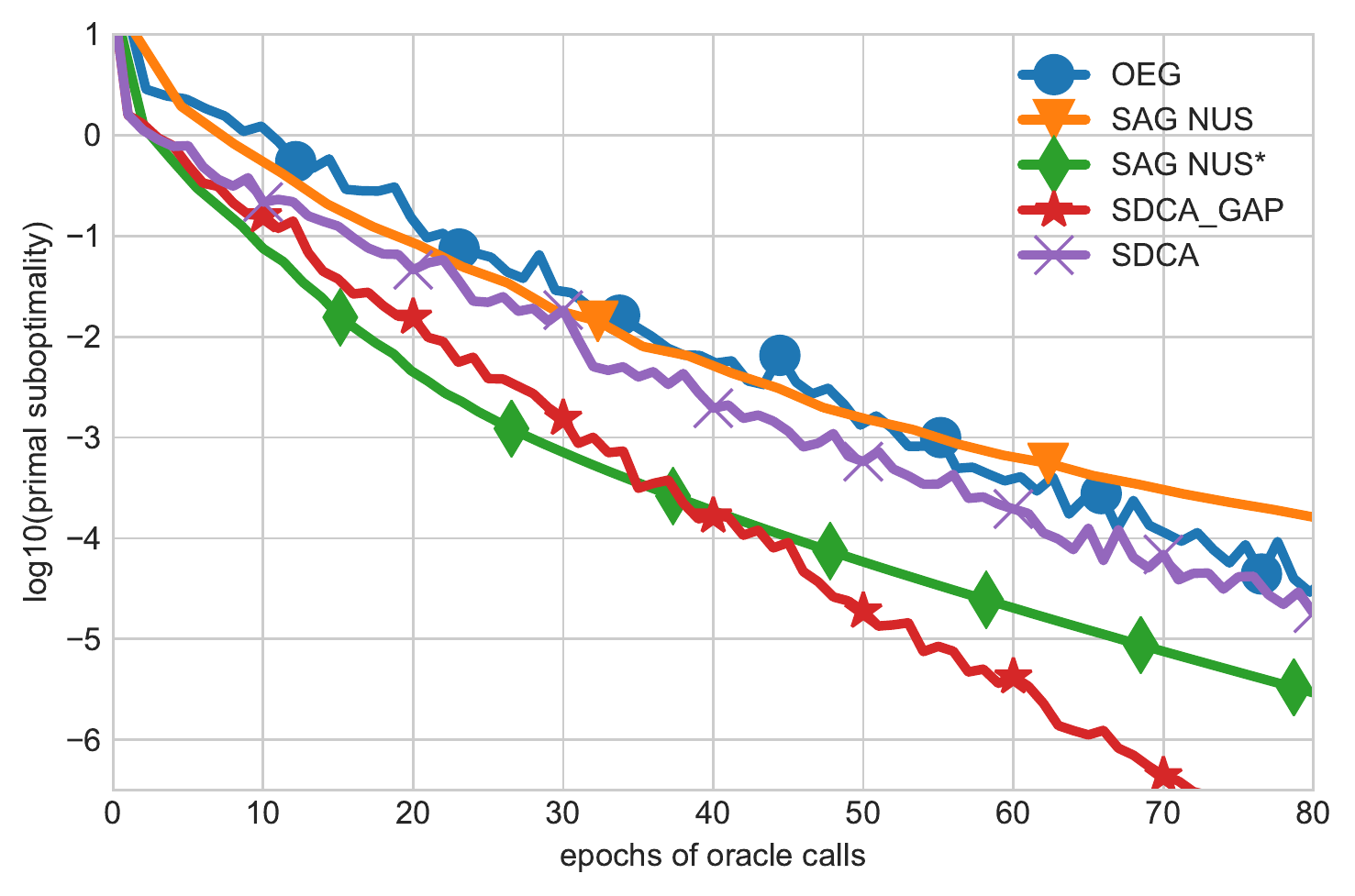}
\caption{OCR (Oracle Calls)}\label{fig:OCR oracles}
\end{subfigure}%
\begin{subfigure}{0.33\linewidth}
\centering
\includegraphics[width=\linewidth]{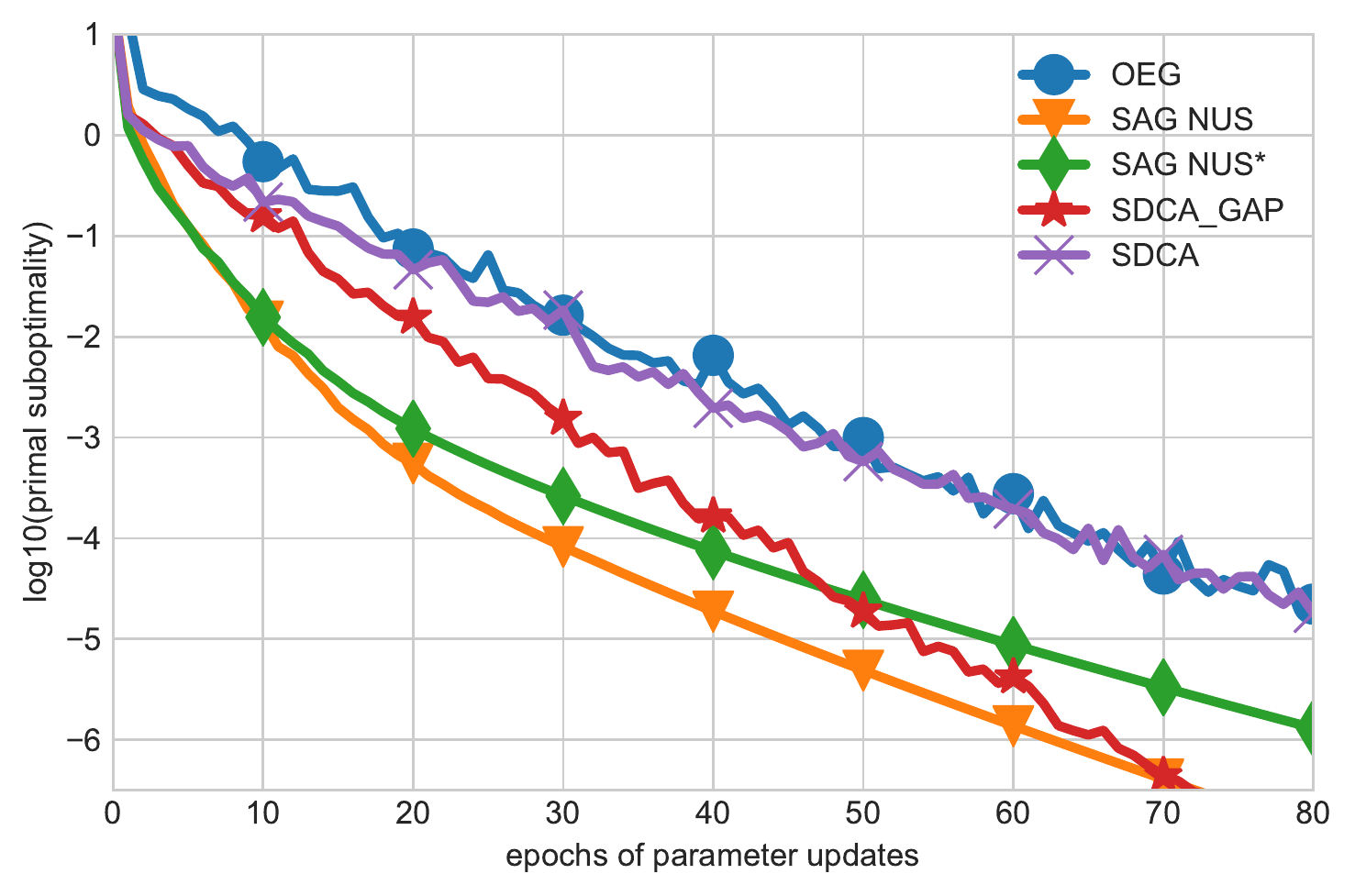}
\caption{OCR}\label{fig:OCR}
\end{subfigure}%
\begin{subfigure}{0.33\linewidth}
\centering
\includegraphics[width=\linewidth]{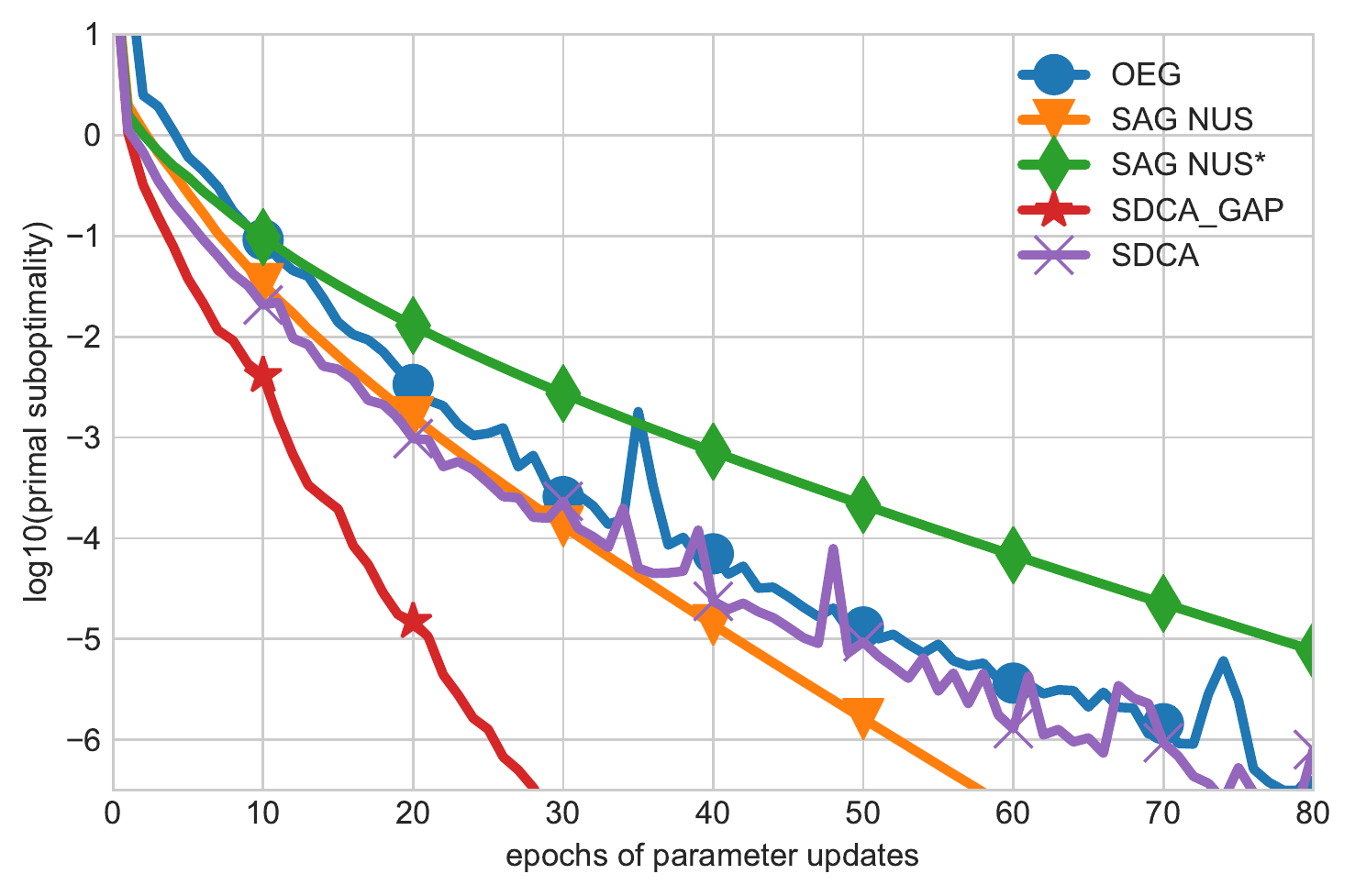}
\caption{CONLL}\label{fig:CONLL}
\end{subfigure} \\
\begin{subfigure}{0.33\linewidth}
\centering
\includegraphics[width=\linewidth]{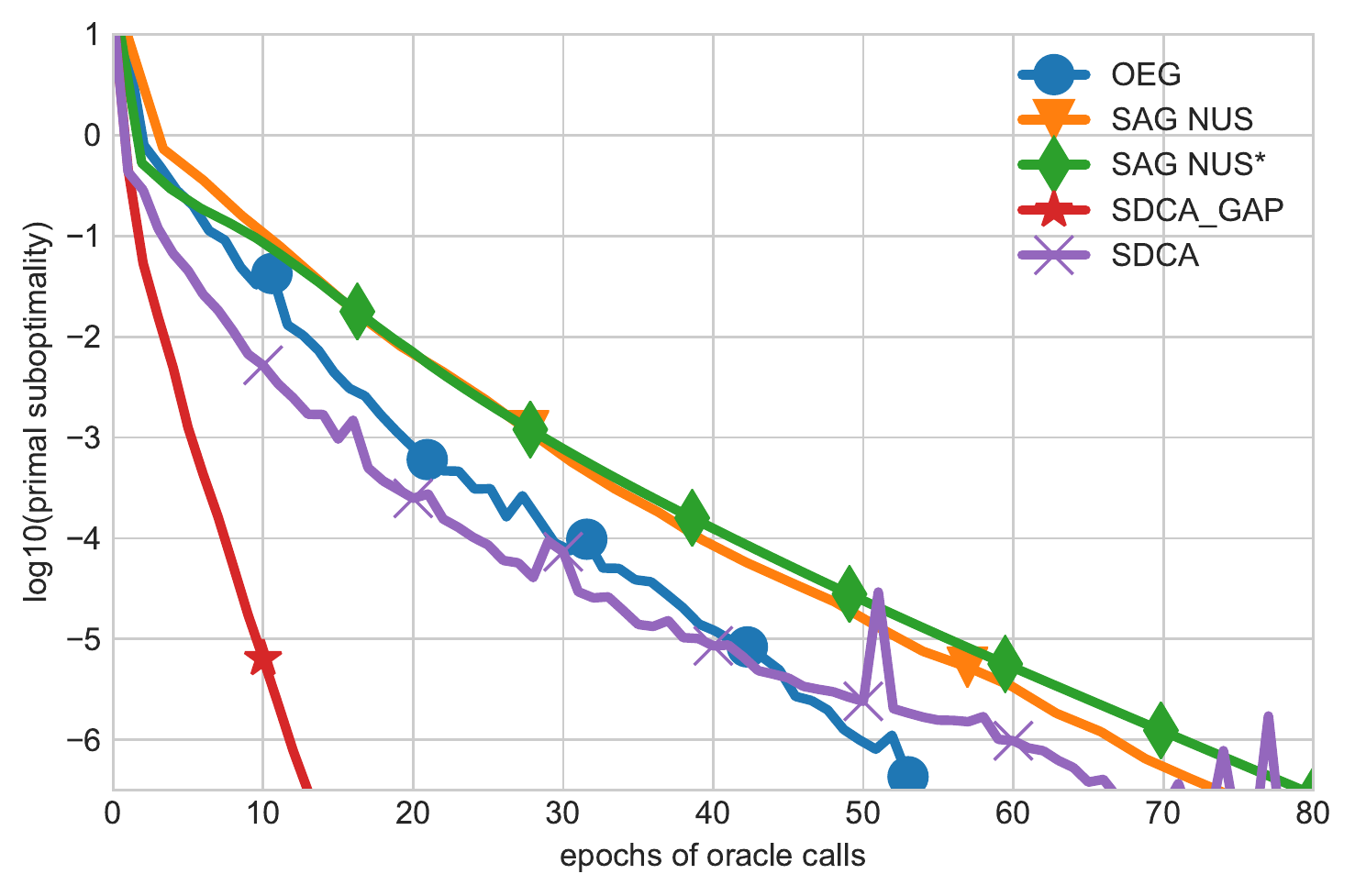}
\caption{NER (Oracle calls)}\label{fig:NER oracles}
\end{subfigure}%
\begin{subfigure}{0.33\linewidth}
\centering
\includegraphics[width=\linewidth]{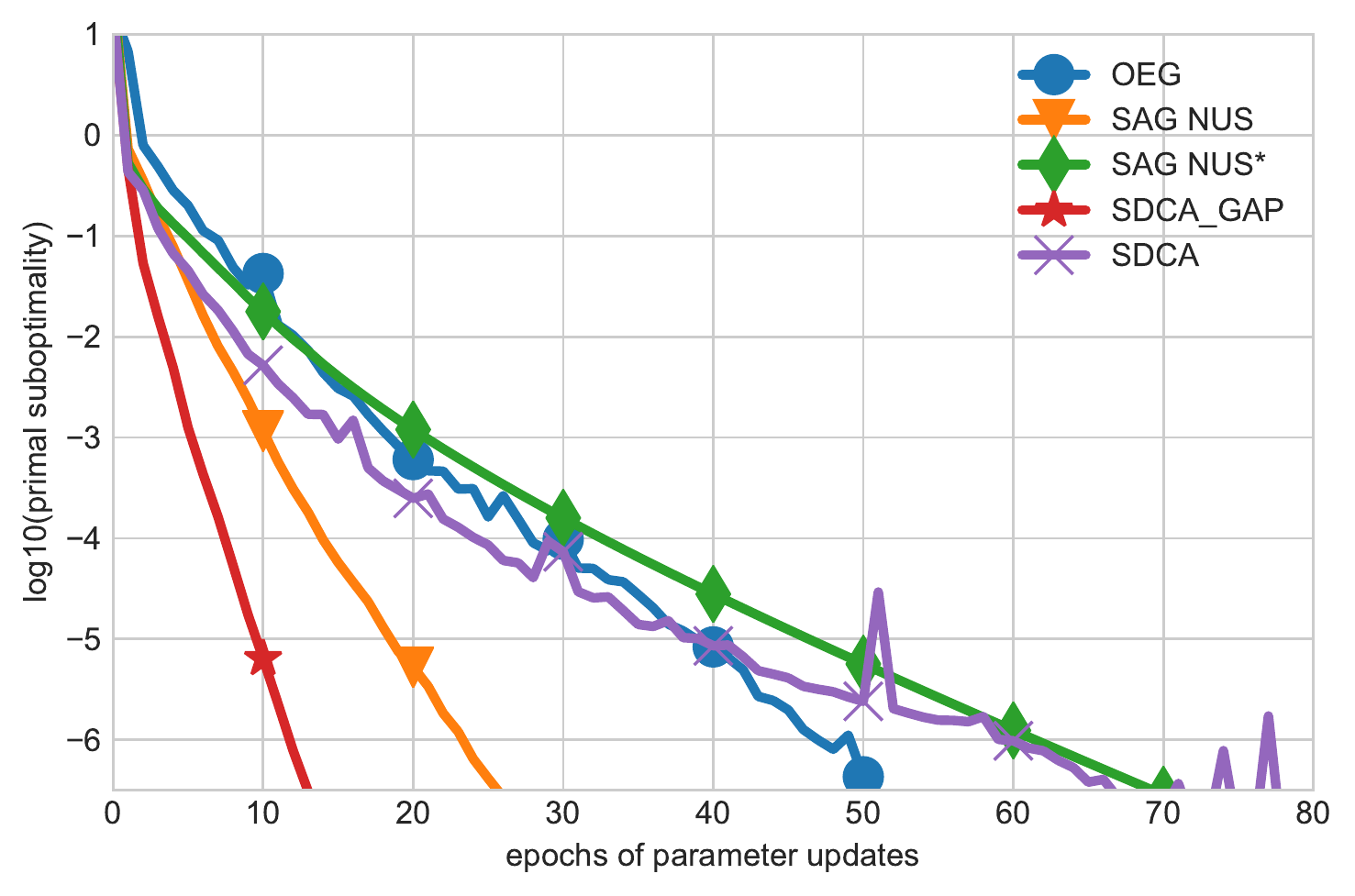}
\caption{NER}\label{fig:NER}
\end{subfigure}%
\begin{subfigure}{0.33\linewidth}
\centering
\includegraphics[width=\linewidth]{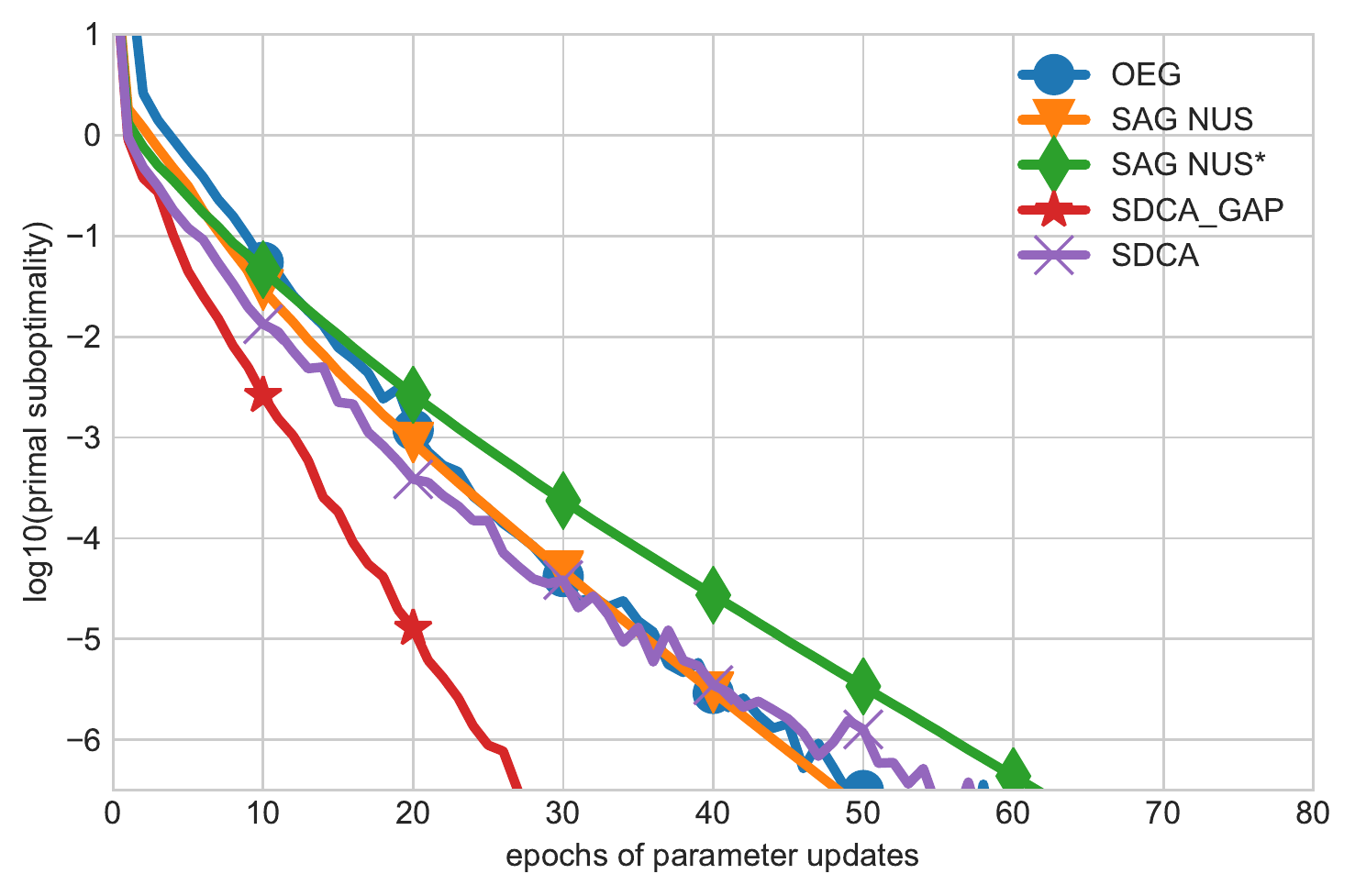}
\caption{POS}\label{fig:POS}
\end{subfigure}%
\caption{Primal sub-optimality as a function of the number of oracle calls (left) or parameters updates (center and right).
	SDCA refers to uniform sampling.
	SDCA-GAP refers to sampling Gap sampling 80\% of the time.
	SAG-NUS performs a line search at every iteration.
	SAG-NUS* implements a line-search skipping strategy.
	It appears worse than SAG-NUS when we look at the number of updates, which hides the cost of the line search.
}
\label{fig:comparison with SAG and OEG}
\end{figure*}

\paragraph{Oracle calls.}
\citet{schmidt2015non} compared the algorithms on the basis of the number of oracle calls.
We report these on OCR and NER in Figures~\ref{fig:OCR oracles} and~\ref{fig:NER oracles}.
Results on the other datasets are in Figure~\ref{fig:oracle_calls} in Appendix~\ref{app:comp_plots}.
This metric was suitable for the methods they compared.
Both OEG and SAG-NUS use a line search where they call an oracle on each step.
SDCA does not need the oracle to perform its line search.
However the oracle is message passing on a junction tree.
It has a cost proportional to the size of the marginals.
Each iteration of the line search require computing the entropy of these marginals, or their derivatives.
These costs are roughly the same.
Comparing the number of oracle calls for each method is thus unfairly advantaging SDCA by hiding the cost of its line search.
It becomes a relevant comparison when a marginalization oracle becomes much more expensive than approximating the entropy (see the discussion in Section~\ref{sec:discussion}).
When this cost is hidden, SDCA-GAP is on par with SAG-NUS* on OCR and it is much faster on the sparse datasets.

\paragraph{Parameter updates.}
To give a different perspective, we report the log of the sub-optimality against the number of parameter updates in Figures~\ref{fig:OCR}, \ref{fig:CONLL}, \ref{fig:NER} and~\ref{fig:POS}.
This removes the additional cost of the line search for all methods.\footnote{
This is a penalty for SAG-NUS* which enforces a line-search skipping strategy.
}

We observe that uniform SDCA and OEG need roughly the same number of  parameters update on all four datasets.
When we add the adaptive gap sampling, SDCA outperforms OEG by a margin.
On OCR, SDCA and SDCA-GAP do not perform as well as SAG-NUS.
On the three other datasets, SDCA-GAP needs less iterations.
In fact, the more sparse the dataset, the less iterations are needed.

This is likely explained by SDCA's ability to almost perfectly optimize each block separately due to its line search method.
More specifically, as the datasets become sparser, the prediction between data points becomes less and less correlated (i.e. the label distribution for two points that share no attributes will not influence each other directly through their primal weights).
In settings where no points share any attributes (completely sparse), all methods optimize each point independently.
SDCA  may perform very well thanks to its precise line search.

In terms of test error, SDCA is on par with SAG, and a bit better than OEG.
All methods reach maximum accuracy after a few epochs.
We report the evolution of the test error in Figure~\ref{fig:test_err} of Appendix~\ref{app:comp_plots}.

Comparing the number of parameters updates also has a disadvantage.
It penalizes methods with line search skipping strategies likes OEG and SAG.
The running time is highly implementation dependent and providing a fair comparison is non-trivial.
We focused on implementation independent comparisons.
SCDA, SAG and OEG have many common operations: the oracle, the computation of the scores and the primal  direction.
The fact that the line search took only 30\% of SDCA's runtime indicates that the conclusion drawn from the number of updates may hold for other metrics.

\section{DISCUSSION} \label{sec:discussion}

In this work, we investigated using SDCA for training CRFs for the first time.
The observed empirical convergence per parameter update was similar for standard SDCA and OEG.
However, SDCA can be enhanced with an adaptive sampling scheme, consistently accelerating its convergence and also yielding faster convergence than SAG with non-uniform sampling on datasets with sparse features.
It would be natural to also implement a gap sampling scheme for OEG, though several quantities needed for the computation are not readily available in standard OEG and would yield higher overhead in actual implementation.
We leave finding a more efficient implementation of a gap sampling scheme for OEG as an interesting research direction.

A key feature of SDCA is to only require one marginalization oracle per line-search.
This could become advantageous over SAG or OEG when the marginalization oracle becomes much more expensive than evaluating the entropy function from the marginals.
Examples for this scenario include: when a parallel implementation is used for the entropy computation; or when the marginalization oracle uses an iterative approximate inference algorithms such as TRW BP whereas an approximation of the entropy is direct from the marginals~\citep{kirshnan15barrierFW}.
Investigating these scenarios with full timing comparison (which is implementation dependent) is a further interesting direction of future work.

We also note that acceleration schemes have been proposed for both SAG and SDCA~\citep{lin2015catalyst, shalev2016accelerated}, though they have not been tested yet for training CRFs.

\subsubsection*{Acknowledgments}
We are thankful to Thomas Schweizer for his numerous software engineering advices.
We thank Gauthier Gidel and Akram Erraqabi who started this project.
We are indebted to Ahmed Touati for his involvement during the first phase of the project.
Alexandre Piché was supported by the Open Philantropy Project.
This work was partially supported by the NSERC Discovery Grant RGPIN-2017-06936.

\renewcommand{\refname}{\subsubsection*{References}}
\bibliographystyle{abbrvnat}
\bibliography{../optimization}


\clearpage
\onecolumn
\appendix

\section{IMPLEMENTATION} \label{app:sec:implementation}
We discuss some practical aspects of SDCA: initialization, memory requirement and how to do the line search.

\subsection{INITIALIZATION}\label{app:ssec:initialization}

As discussed in~\citet{schmidt2015non}, the initialization of dual methods for CRFs can influence significantly their performance.
We describe here a motivation for a suggested good initialization for $\balpha$.
Suppose that we put all the mass for $\alpha_i$ on the ground truth label $y_i$, i.e. $\alpha_i = \delta_{y_i}$ where $\delta_y$ is the Kronecker delta function on $y$ -- this represents the ``empirical distribution'' on one example.
Let $\bm\delta$ be the concatenation $(\delta_{y_i})_{i=1} ^n$.
Similarly, let $\bu$ be the concatenation of the uniform distribution on the labels for each training example.
We have the following chain of relationships:
\begin{align*}
	\bm\delta & \quad \xrightarrow{~\hat w~} & \bm{0} & \quad \xrightarrow{~\hat \alpha~} & \bu & \quad \xrightarrow{~\hat w} \dots \\
	\cD(\bm\delta)=0 &  & \text{small } \cP(\bm{0}) &  & \cD(\bu) &
\end{align*}

What is important here is that $\cP(\bm{0})$ is small.
If each node can take up to $K$ values, and there are $n$ sequences for a total of $N$ nodes, $\cP(\bm{0}) = \frac{N}{n}\log(K)$.
On all our datasets this is below 100.
This means that using $\bm{\alpha}^{(0)}=\bm\delta$ gives an initial duality gap equal to $\cP(\bm{0})\lesssim10^2$.
In contrast, using $\bm{\alpha}^{(0)} = \bu$ as used in the original OEG code\footnote{\texttt{egstra-0.2} available online at \url{http://groups.csail.mit.edu/nlp/egstra/}. This is also the initialization used in the main text of~\citet{schmidt2015non}.} consistently gave extremely large $\hat w(\bu)$ resulting in a large negative dual score and large primal score, and raising numerical stability issues.
Primal methods usually initialize their weights to zero.
The dual counter part is the empirical distribution because it yields the same primal vector and score.
For these reasons, we ideally would like to use $\bm\delta$ as the initialization.

There is catch though. On the borders of the simplex, the entropy has infinite gradient and curvature.
This is a bad behavior if we wish to use this information for the line search.
A natural strategy to mitigate this effect is to take a (small $\epsilon$) convex combination with the uniform:
\begin{equation} \label{eq:initialization:app}
	\balpha^{(0)} := \varepsilon \bu + (1-\varepsilon) \bm\delta \, .
\end{equation}
This is what we use in our experiments.
Graphically, the initial point will be on a segment between a corner of the simplex and the center.
This is the same initialization that \citet[App.~D of the Sup.~Mat.]{schmidt2015non} discovered empirically. It was also used implicitly by~\citet{collins2008exponentiated} when they took the regularization path approach by starting the method with a very large regularization parameter~$\lambda$.

\subsection{MEMORY REQUIREMENT}

Variance reduced methods use memory (except SVRG) to control the variance of the update.
This memory cost can be quite large as it grows linearly with the size of the dataset.
\citet{schmidt2015non} suggested a smart way to reduce this memory cost for  SAG : for a sequence with hand crafted features, one stores only the unary marginals and the binary features.
There is no such trick for dual methods, and both OEG and SDCA have to store the full marginals.
It turns out that if each node can take $K$ values, we have to allocate about $K$ times more memory than for SAG.
This can become a problem: for our larger dataset, part of speech tagging on Penn-Tree Bank Wall-Street Journal, we needed about 15GiB of RAM.

\subsection{LINE SEARCH} \label{app:sec:implementation line search}
The line search is an important part of the algorithm.
Each evaluation of the line search function or its derivatives is quite expensive.
We need to aggregate values from the whole marginal which has a size $\sum_c | \cY_c |$ (though this can be done in parallel).
As a comparison, running the sum-product algorithm over the junction tree has a cost $2\sum_c | \cY_c |$ (though this is a sequential algorithm).
There are other overhead in the algorithm such as computing the scores $\bw^T F_c(x, y_c)$ or estimating the primal direction $A_i \delta_i$, so this is not totally critical.

Yet we wish to reduce the number of function evaluation.
A good way to do so is to use the Newton-Raphson algorithm.
But this uses the first and second derivatives of the line search objective, and the entropy has infinite slope and curvature on the borders of the simplex.
To avoid numerical instability issues, we have to use and store the logarithm of the marginals (as was done for OEG~\citep{collins2008exponentiated}).
We report an empirical study of the line search performance in section \ref{experiment line search}.

\section{DESCRIPTION OF THE FEATURE MAP $F$}
\label{app:feature}

\begin{wrapfigure}{r}{3.5cm}
\centering
    \includegraphics[height=7cm]{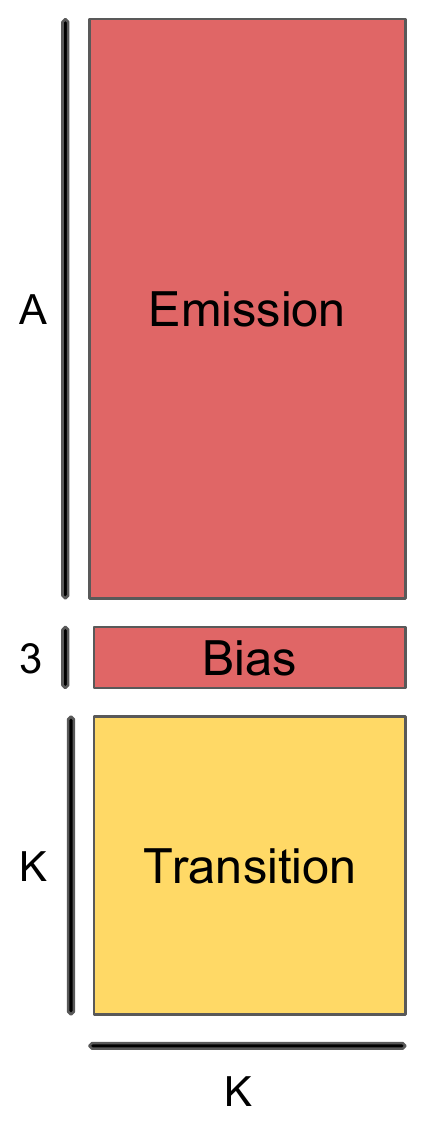}
    \caption{Sketch of the feature map. K is the number of different labels for one node. A is the number of attributes.}
    \label{fig:sketch feature map}
\end{wrapfigure}

The feature map has the same structure on all the data sets (cf Figure~\ref{fig:sketch feature map}).
We first draw the distinction between unary features (in red) and binary features (in yellow). The features can be written as the sum of the unary and binary features:
\begin{equation*}
	F(x, y) = \sum_{t=1}^T F_t(x_t, y_t) +  \sum_{t=1}^{T-1} F_{t, t+1}(y_t, y_{t+1}).
\end{equation*}
Unary Features depend only on the label of one node $y_t$ and the corresponding data point~$x_t$: $F_t(x_t, y_t)$.
Binary features depend only on the labels of two neighboring nodes : $F_{t, t+1}(y_t, y_{t+1})$.
It is a design choice not to directly model the relationship between two neighboring data points, e.g. $F(x_t, x_{t+1}, y_t, y_{t+1})$.
In practice the binary features simply count the number of transitions between $y_t$ and $y_{t+1}$, hence the yellow square.

For unary features, it is a bit more complex.
For each data sequence $x$, we extract an embedding for each position $t$, $\varphi(x, t)$.
For OCR, it is simply the 128 pixels image itself $\varphi(x, t) = x_t$.
For the language tasks, it is a count of the appearance of certain attributes, e.g, what is the word $x_t$, what are the words at position $t-1$, $t+1$, and so on.
A complete list of the attributes is available at \url{http://www.chokkan.org/software/crfsuite/tutorial.html}.
For each word (=node), between $13$ and $20$ features are extracted depending on the dataset.
In total the number of different attributes extracted ranges from $73,000$ to $300,000$, hence the sparsity of the features.
 We denote $A$ the number of attributes, or alternatively the size of the embedding.
 For each node with point $x_t$ and label $y_t$, $F_t(x_t, y_t)$ puts the embedding $\varphi(x, t)$ in the column indexed by $y_t$ of the red emission matrix.
 In this same column, we add some bias.
The bias part has 3 dimensions.
The first component counts the appearance of the label.
The second component counts the appearance of the label in first position of a sequence, ($t=0$).
The last component counts the number of appearance in the last position of a sequence.

\section{HOW TO COMPUTE THE RADIUS OF THE FEATURES}
\label{app:radius}

We drop the i index for now. We look at the pair $(x, y)$. We want to evaluate an upper bound on:
\beq
R = \|A\|^2_{1\rightarrow 2}  = \max_{y\in\cY} \| \psi(y) \|_2^2 = \max_{\tilde y\in\cY} \| F(x, y) - F(x, \tilde y)\|_2^2.
\label{app:eq:radius}
\eeq
We are using the special nature of the features to estimate this radius.
Remark that in the standard feature maps that we used (Appendix~\ref{app:feature}), there is one column per label.
If the label $y_t$ is assigned to the node $t$, then all the features extracted from that node are inserted in the column associated to $y_t$.

How to build a $\tilde y$ maximizing the distance between features?
First we build the ground truth features : $F(x, y)$.
Then we look at the labels included in the sequence $y_t$.
In each data set, the $K$ labels never appear together in one sequence.
We find a label $z$ that does not appear in the original sequence.
Then a sequence $\tilde y$ maximizing the objective~\eqref{app:eq:radius} is the sequence composed only with that label $z$.

Why? There are two reasons.
First, $F(x, y) \geq 0 \, \forall (x, y)$ thus we want $F(x, y)$ and $F(x, \tilde y)$ to have disjoint supports such that the radius can be written as:
\beq
	R = \|F(x, y) \|^2 + \|F(x,\tilde y)\|^2.
\eeq
Second, we want to maximize $\|F(x,\tilde y)\|^2$.
We need to put all the weights on few coordinates, instead of dispersing it.
This is because we look at the $\ell^2$ norm.
For the $\ell^1$ norm there would be no difference.
By repeating only one label, we effectively concentrate all the weights in one column.

Following the steps described above, we can evaluate the radii for the whole data set.

\section{A CONVERGENCE BOUND ON THE DUALITY GAP}\label{app:bound duality gap}
It turns out that any algorithm with a convergence bound on the primal or the dual sub-optimality for problems~\eqref{eq:primal_problem} and~\eqref{dual problem}, can transpose it to a convergence bound on the duality gap.
That will be at the cost of a constant.
To transpose a result of the \textit{primal} sub-optimality to the duality gap, one can go by the norm of the gradient using the smoothness of~$\cP$, that we denote $L$:
\beq
 \cP(\bw) - \cP(\bw^*)\geq \frac{1}{2L}\|\nabla \cP(\bw) \|^2 \stackrel{\eqref{eq:gradientGap}}{=} \frac{\lambda}{L} g(\bw, \hat \alpha(\bw)).
\eeq
The first inequality above is a standard one from convex analysis for convex functions with Lipschitz-continuous gradients (see e.g. \citep[eq. (2.1.6)]{nesterov_introductory_2004}).
Whatever bound we get on the primal sub-optimality, we can translate it to the duality gap by losing a constant $L/\lambda \geq \kappa$, where $\kappa$ is the condition number.

To transpose a result from the \textit{dual} sub-optimality to the duality gap, one can use the uniform ascent lemma, Equation~\eqref{app:eq:uniform ascent lemma} from Appendix~\ref{app:proof}:
\beq
	\cD(\balpha^*)  - \cD(\balpha)
	\geq \E [\cD(\balpha^+)]  - \cD(\balpha) 
	\geq \frac{s}{n} g(\hat w(\balpha), \balpha)
\eeq
where the expectation is taken over the stochasticity of the update.
Let us look at this new constant.
We know that $1/s = 1+ \frac{R}{n \lambda \strgconvex}$.
We can relate it to the smoothness $L  \approx \lambda + \frac{R}{\strgconvex}$.
This time we lose a factor $n/s \approx n + \frac{L}{\lambda} \geq n + \kappa$.
For a well-conditioned problem ($n \gg \kappa$) this is much larger than the constant we lose from the primal to the gap.

\section{ADDITIONAL COMPARISON PLOTS}
\label{app:comp_plots}

We provide additional figures on the primal sub-optimality as a function of oracle calls (Figure \ref{fig:oracle_calls}), the test error as a function of epochs (Figure \ref{fig:test_err}), the  impact of reducing the precision of the Newton line-search (Figure \ref{fig:subprecision}) and the ratio between the estimate of the duality gap and the ground truth (Figure \ref{fig:ratio}).

\begin{figure}[H]
\centering
\begin{subfigure}{0.4\linewidth}
\centering
\includegraphics[width=\linewidth]{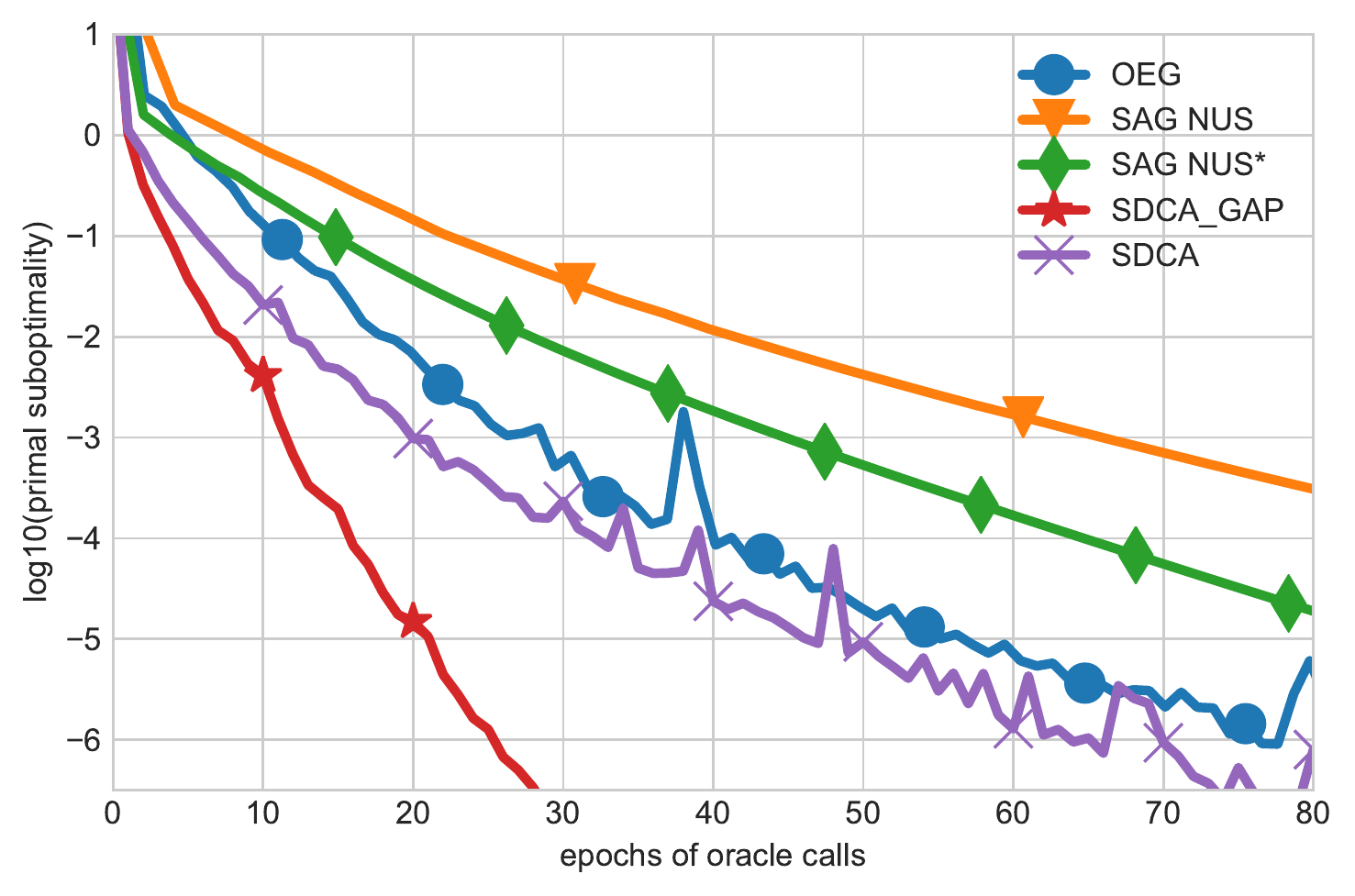}
\caption{CONLL}
\end{subfigure}
\begin{subfigure}{0.4\linewidth}
\centering
\includegraphics[width=\linewidth]{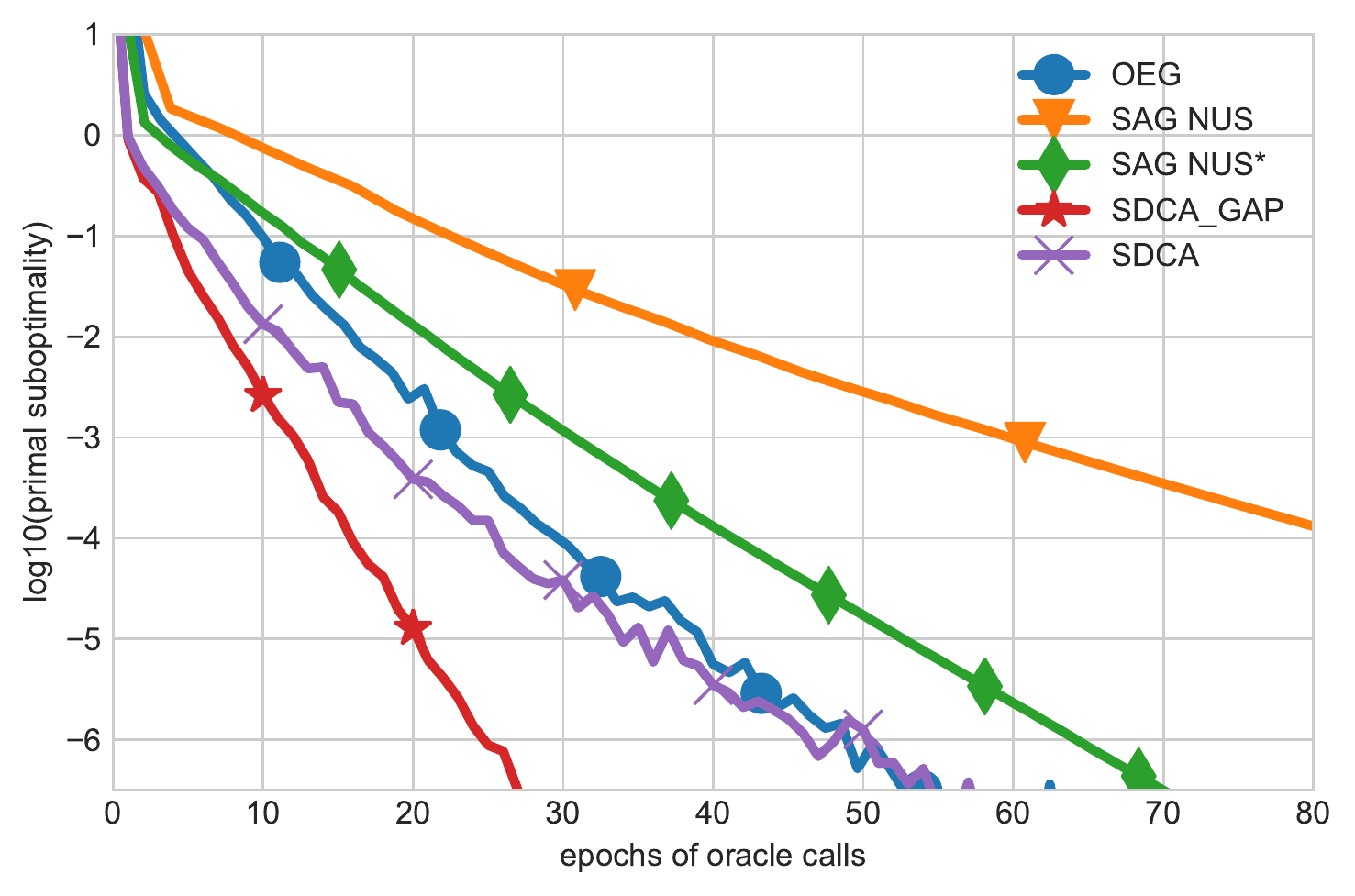}
\caption{POS}
\end{subfigure}%
\caption{Primal sub-optimality as a function of the number of oracle calls. SDCA-GAP performs much better than the competing methods for this metric partly because its line search does not require oracle calls.}\label{fig:oracle_calls}
\end{figure}

\begin{figure}[H]
\centering
\begin{subfigure}{0.4\linewidth}
\centering
\includegraphics[width=\linewidth]{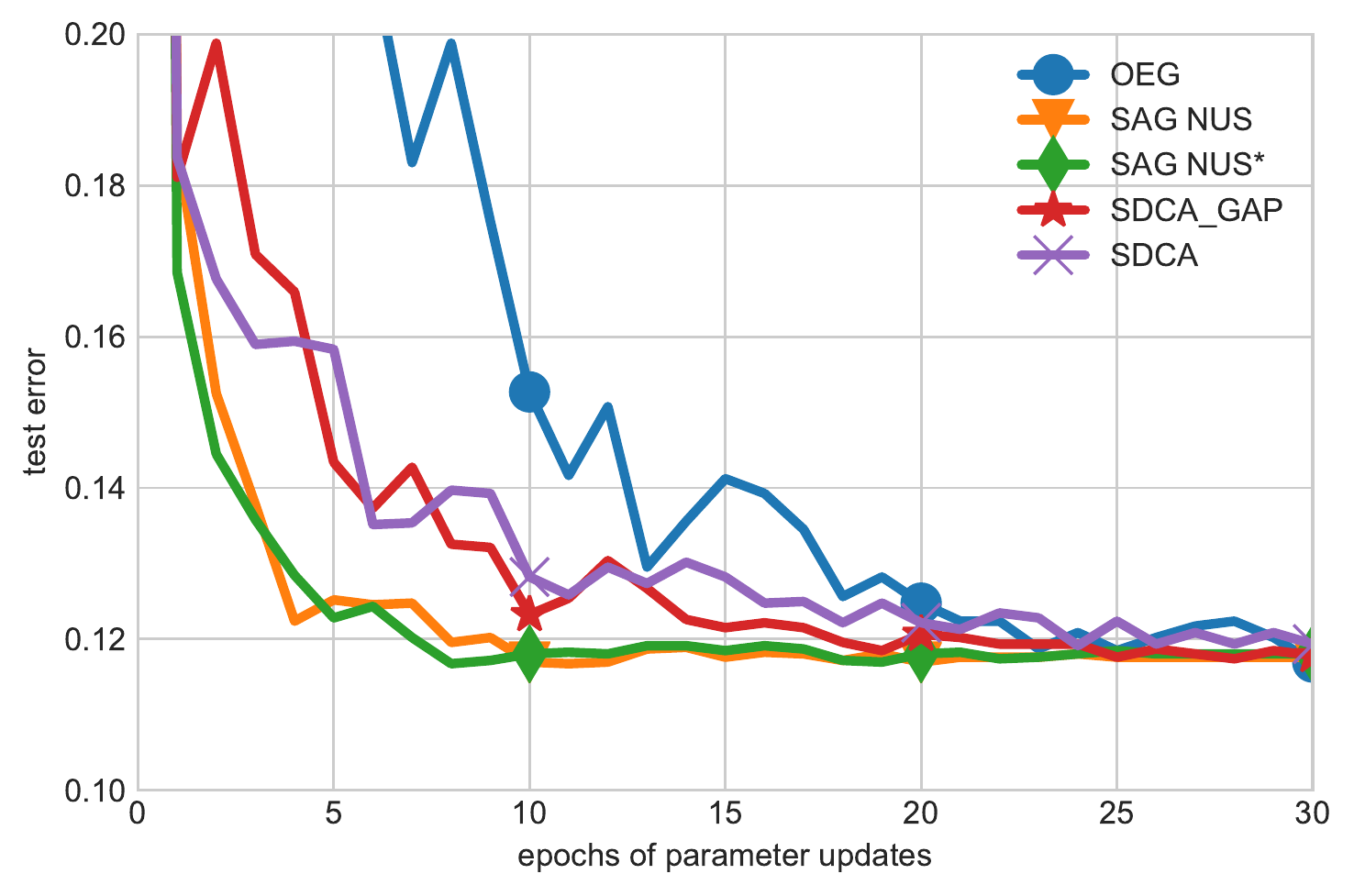}
\caption{OCR}
\end{subfigure}%
\begin{subfigure}{0.4\linewidth}
\centering
\includegraphics[width=\linewidth]{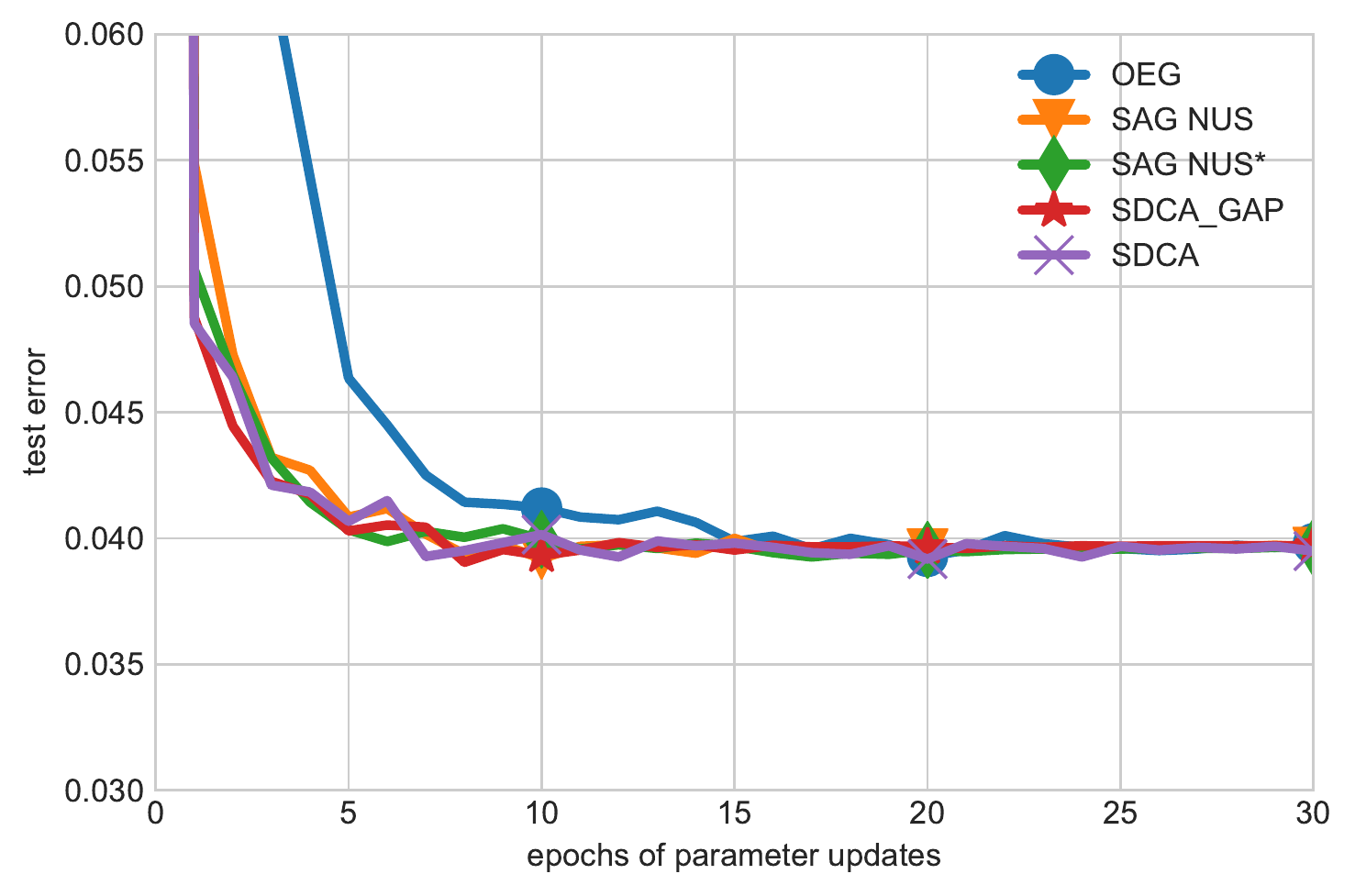}
\caption{CONLL}
\end{subfigure} \\
\begin{subfigure}{0.4\linewidth}
\centering
\includegraphics[width=\linewidth]{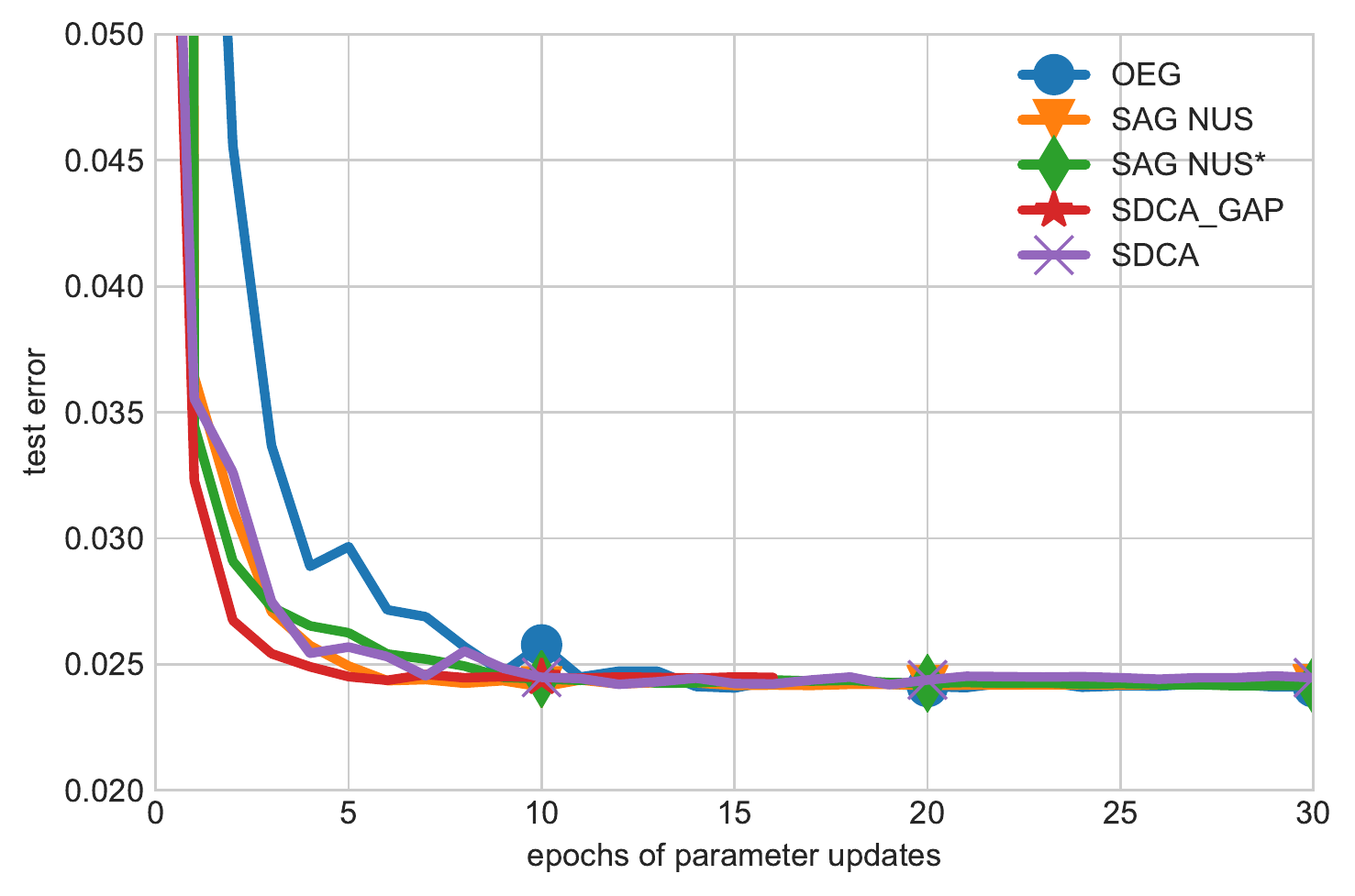}
\caption{NER}
\end{subfigure}%
\begin{subfigure}{0.4\linewidth}
\centering
\includegraphics[width=\linewidth]{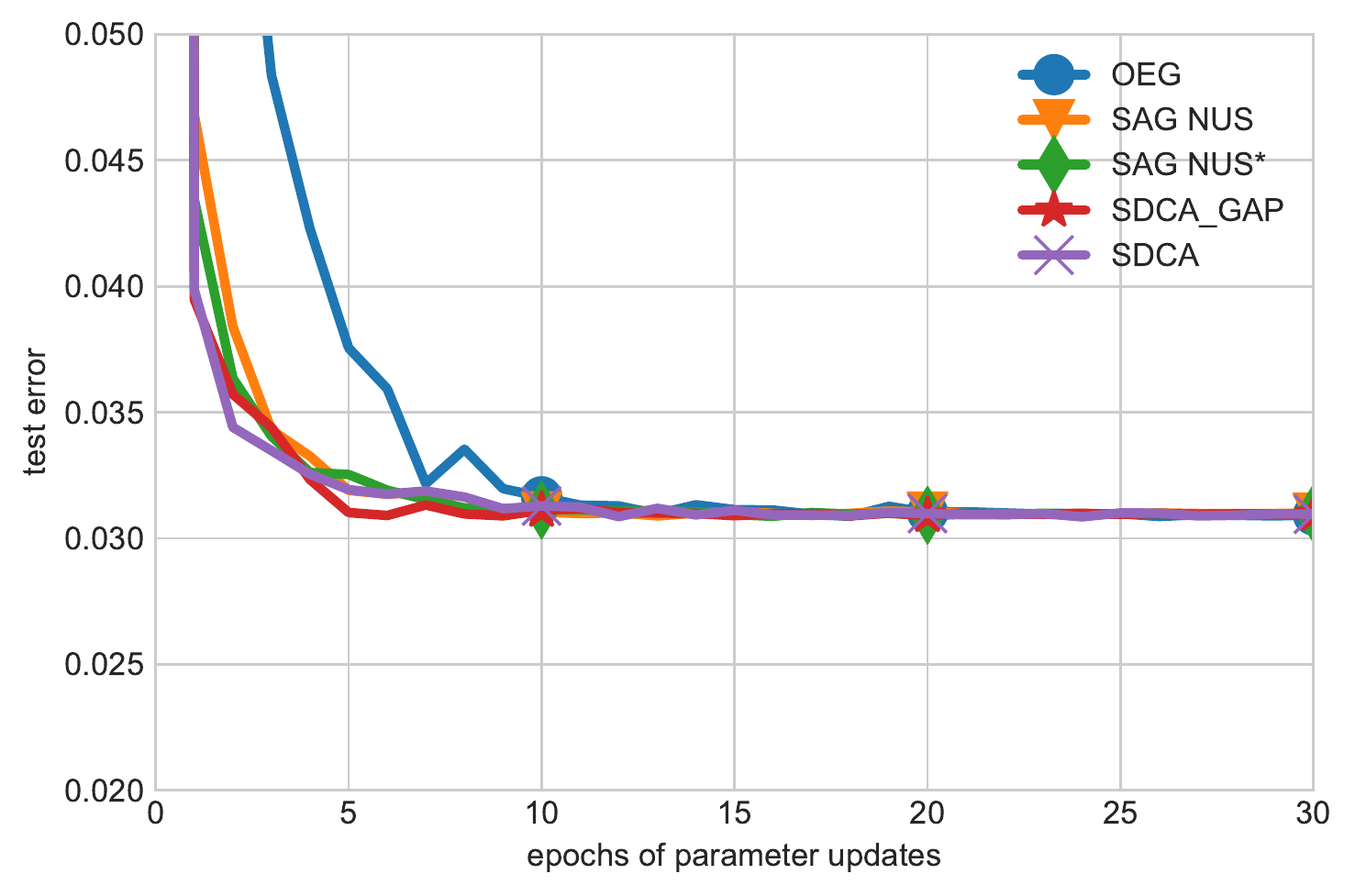}
\caption{POS}
\end{subfigure}%
\caption{Test error against number of epochs. Every methods reach the same test error. SDCA and SAG have the same convergence speed.}\label{fig:test_err}

\end{figure}

\begin{figure}[H]
    \centering
    \begin{subfigure}{0.4\linewidth}
        \centering
        \includegraphics[width=\linewidth]{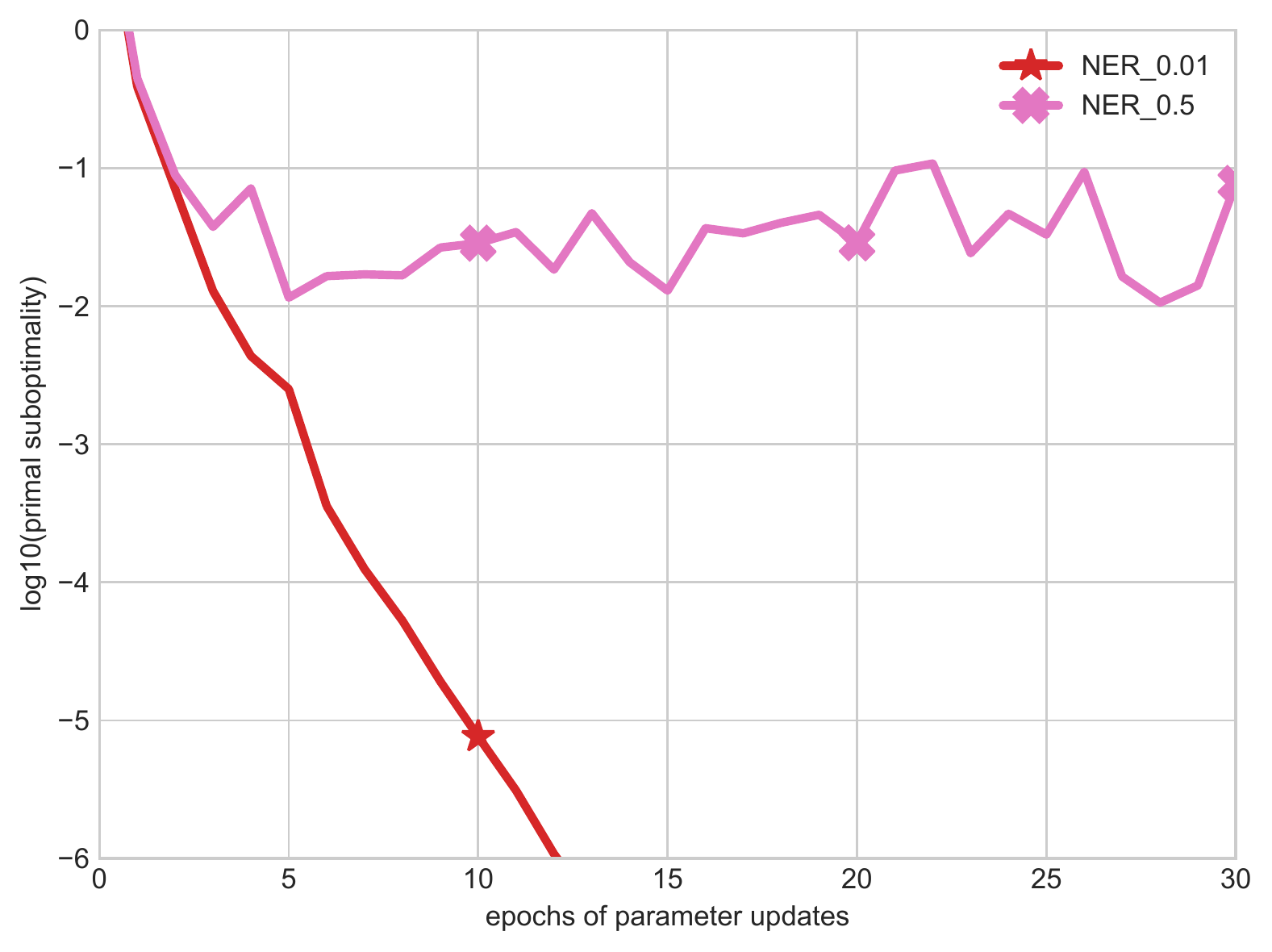}
        \caption{NER}
    \end{subfigure}
    \begin{subfigure}{0.4\linewidth}
        \centering
        \includegraphics[width=\linewidth]{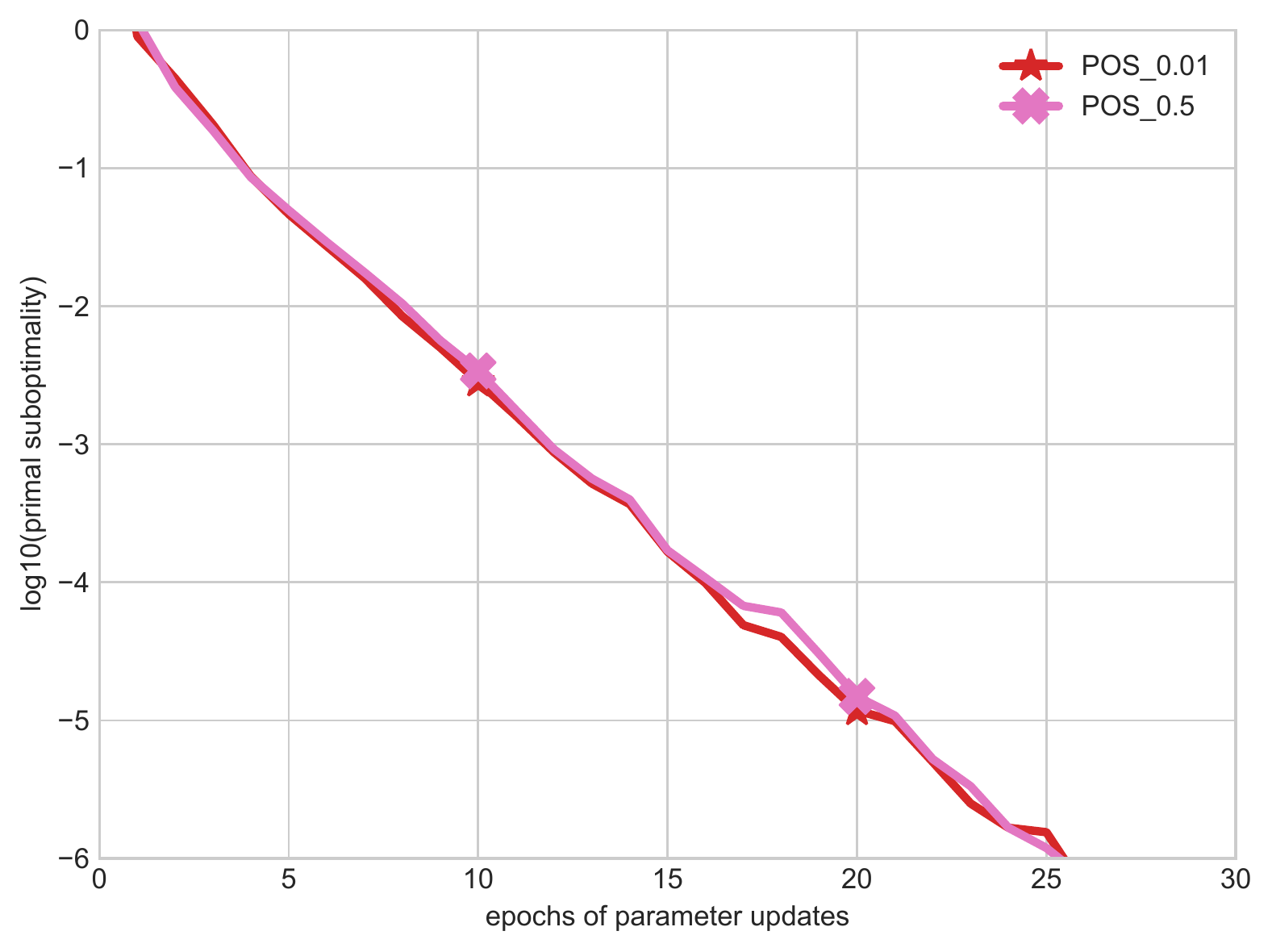}
        \caption{POS}
    \end{subfigure}
	\caption{
		Performance of SDCA on NER and POS with a Newton line-search.
		The number after the name of the dataset indicates the sub-precision we asked.
		A sub-precision of 0.5 effectively means that Newton stops after 1 step.
		While there is no difference between the curves for POS, 1 step of Newton update fails to converge on NER.
	}\label{fig:subprecision}
\end{figure}

\begin{figure}[H]
\centering
\begin{subfigure}{0.4\linewidth}
\centering
\includegraphics[width=\linewidth]{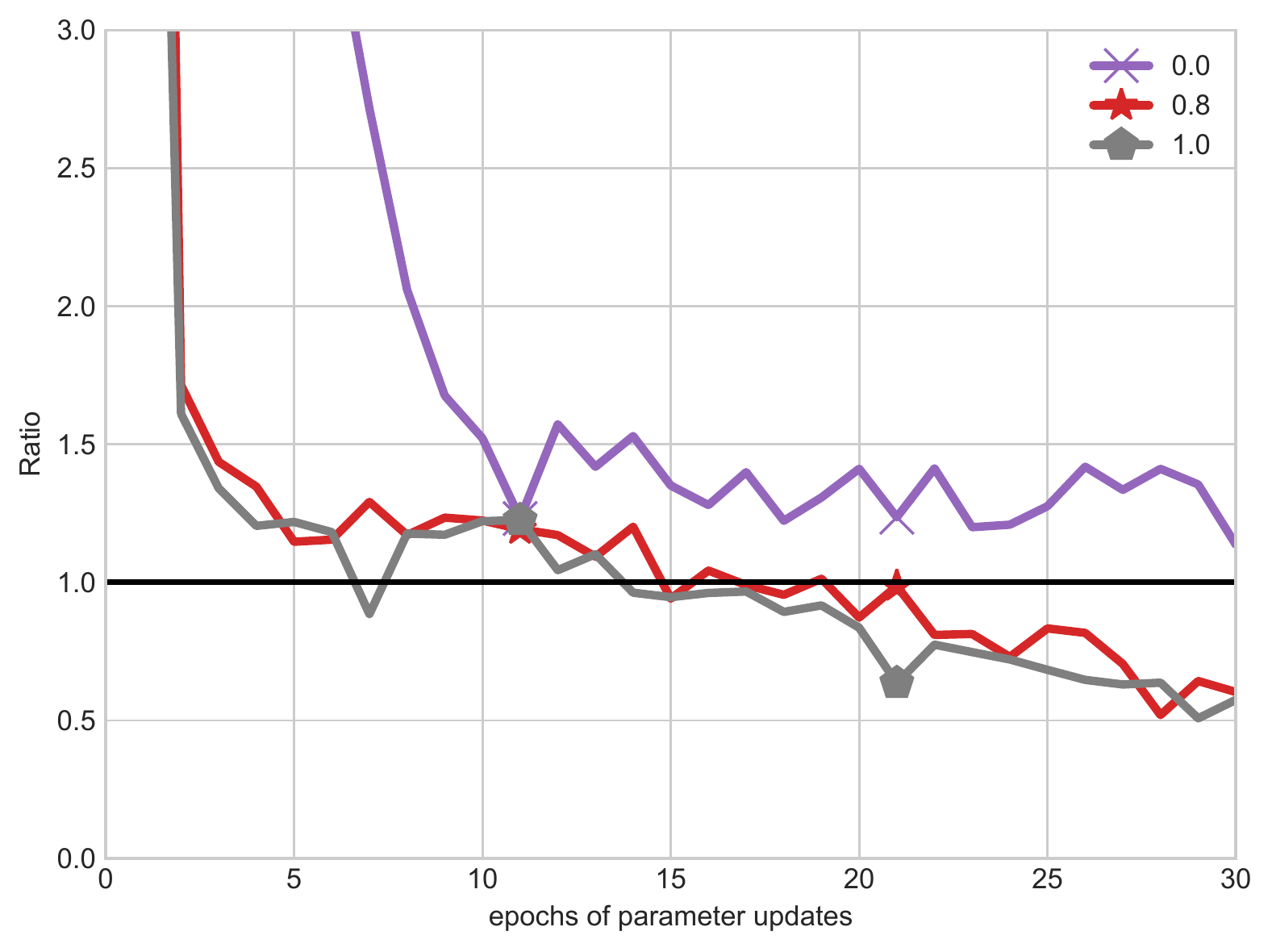}
\caption{CONLL}
\end{subfigure}
\begin{subfigure}{0.4\linewidth}
\centering
\includegraphics[width=\linewidth]{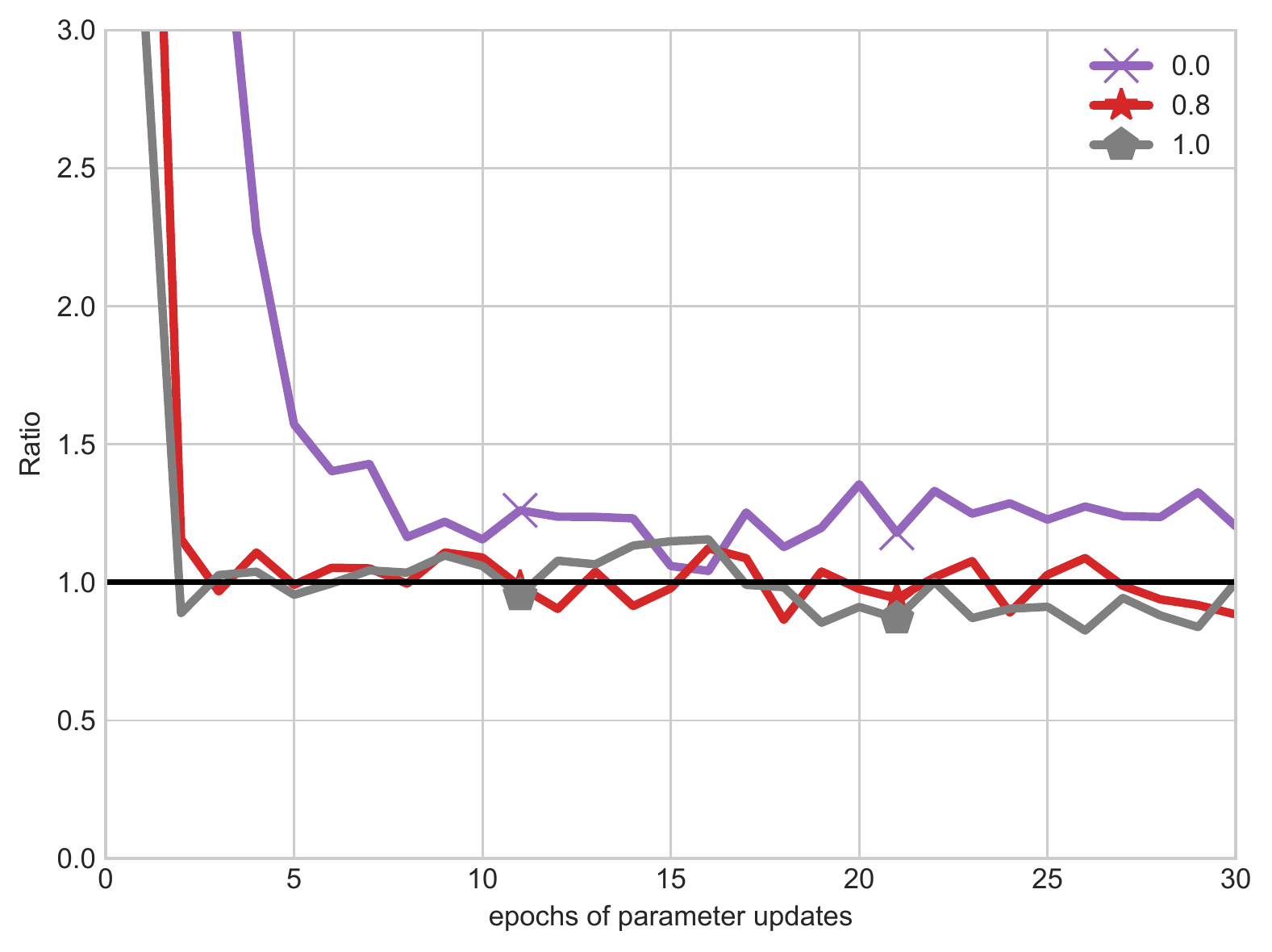}
\caption{OCR}
\end{subfigure}
\caption{The ratio between the estimate of the duality gap and the ground truth as a function of the proportion of non uniform sampling. The gap sampling tends to underestimate this value, whereas the uniform sampling tends to over-estimate it.}\label{fig:ratio}

\end{figure}

\section{A TECHNICAL REPORT ON NON-UNIFORM SAMPLING FOR STOCHASTIC DUAL COORDINATE ASCENT}
\label{app:nusampling}

In this section, we review the proofs of convergence of SDCA and its variants with importance and residual sampling.
Then we derive bounds on the convergence rate of two new sampling scheme for SDCA.
The first scheme samples proportionally to the duality gaps of each individual variable.
The second scheme is similar to the first one, but it corrects the duality gaps with the Lipschitz constant of the primal problem.

\subsection{SETTING}
We derive these bounds in a more general setting than the logistic regression, and we have to introduce some new notation.

Let $\bw$ denote the weights vector parameter, and $A_i$ the i-th features matrix.
Let $\phi$ be the primal loss function.
We suppose it is convex and $1/\strgconvex$-smooth with respect to $\|.\|_P$ (dual norm $\|.\|_D$).
The regularizer $r$ is supposed 1-strongly convex with respect to $\|.\|_{P'}$ (dual norm $\|.\|_{D'}$).
Because $\phi$ and $r^*$ are smooth, they are also differentiable. Note that every starred variable represent its dual conjugate.

The empirical loss minimization problem is:
\begin{equation}
	\label{app:primal problem}
	(P) \quad \min_{\bw \in \real^d} \lambda r(\bw) + \frac{1}{n} \sum_{i=1}^n \phi_i(-A_i^T \bw).
\end{equation}
Its Fenchel dual problem is:
\begin{equation}
	\label{app:dual problem}
	(D) \quad \max_{\balpha | \forall i, \alpha_i \in \dom \phi^*} -\lambda r^*(\hat v(\balpha)) - \frac{1}{n} \sum_{i=1}^n \phi_i^*(\alpha_i),
\end{equation}
with
\begin{equation}
	\label{app:dual to primal}
	\hat v(\balpha) := \frac{1}{\lambda n} \sum_i A_i \alpha_i \quad \textrm{and} \quad \hat w(\balpha) \in  \nabla r^* (\hat v(\balpha)).
\end{equation}
We also note:
\begin{equation}
	\label{app:primal to dual}
	\forall i, \beta_i = \hat \alpha_i(\bw) \in  \nabla \phi_i(-A_i^T \bw).
\end{equation}
Minimization of the empirical risk can often be interpreted as going around the diagram below.
\begin{displaymath}
    \xymatrix{ \bm w \ar[r] &   \nabla \phi_i(-A_i^T \bm w) \ar[d] \\
                \nabla r^* \big ( \frac{1}{\lambda n} A \balpha \big ) \ar[u] &  \balpha \ar[l] }
\end{displaymath}

We define the squared radius of the features for a given sample i as the operator norm of the matrix $A_i$:
\begin{equation}
	R_i := \|A_i\|^2_{D\rightarrow D'}.
\end{equation}
We also define the maximum squared radius as $R=\max_i R_i$ and the mean radius $\bar R = \frac{1}{n} \sum_i R_i$.

\paragraph{Log-likelihood special case.}
The loss $\phi(z)=\log(\sum_y \exp(z_y)) $ is 1-smooth with respect to the max-norm.
Its convex conjugate is the negative entropy $\phi^*(\alpha) = -H(\alpha) = \sum_y \log(\alpha_y) \alpha_y $ which is in turn 1-strongly convex with respect to the $\ell_1$-norm, and whose domain is the simplex.
We use the $\ell_2$ regularization whose dual function is itself.
We thus have $R_i =  \|A_i\|^2_{1\rightarrow 2} = \max_y \| \psi_i(y) \|_2^2$.
We also have a special expression for the primal to dual function $\beta_i = p(.|x_i; \bw) \propto \exp(-\bw^T \psi_i(.))$.
The dual variable is obtained as the conditional probability of the primal model.
Conversely, the primal weights are obtained as the expectation of the features $\psi_i(y)$, which are the columns of $A_i$.

\subsection{DUALITY GAPS}
We derive an interesting form on the duality gaps that support a new sampling strategy.
This is not needed to understand the convergence rates of SDCA and its variants, and the reader may skip this section.

The duality gap is:
\beq
g(\bw,\balpha) = P(\bw)-D(\balpha) = \lambda \left( r(\bw)+r^*(\frac{ A \balpha}{\lambda n}) \right) + \frac{1}{n}\sum_{i=1}^n \phi(-A_i^T\bw) +\phi^*(\alpha_i).
\eeq

Because of the two conjugate pairs $(r, r^*)$ and $(\phi, \phi^*)$ there are two apparent ways to simplify it.
One is to take the conjugate primal variable $\bw \vcentcolon= \hat w (\balpha)$, another is to take the conjugate dual variable $\balpha \vcentcolon= \hat \alpha(\bw)$.

\paragraph{Conjugate primal variable.}
Under the hypothesis $\bw = \hat w (\balpha)$, we obtain:
\beq
r(\bw)+r^*(\frac{ A \balpha}{\lambda n}) = \bw^T \frac{ A \balpha}{\lambda n}.
\eeq
The duality gap simplifies:
\beq
g(\hat w(\balpha), \balpha)
= \frac{1}{n}\sum_{i=1}^n \phi(-A_i^T \hat w(\balpha)) + \phi^*(\alpha_i) - \alpha_i^T (-A_i^T \hat w(\balpha))
= \frac{1}{n}\sum_{i=1}^n F_\phi(-A_i^T \hat w(\balpha), \alpha_i),
\eeq
where $F_\phi(s, \alpha)$ is the Fenchel duality gap~\eqref{eq:Fench} between vectors $s$ and $\alpha$.
When $\phi$ is the log-sum-exp, these vectors are the score (or logit) $s$ and the probability $\alpha$.
We want to simplify this further to directly relate $\balpha$ and its next iterate $\hat\alpha_i\circ\hat w(\balpha)$.
To do so we need another condition:
\beq
\langle \nabla \phi^* \circ \nabla \phi(s) - s , \beta - \alpha \rangle = 0,
\label{eq:condition_legendre_couple}
\eeq
for all $s\in \dom \phi$ and $\alpha, \beta \in \dom \phi^*$.
Geometrically, the pairs $(s, \nabla \phi^* \circ \nabla \phi(s) )$ should always be aligned orthogonally to $\dom \phi^*$.
This condition~\eqref{eq:condition_legendre_couple} is true whenever $\nabla \phi^* \circ \nabla \phi = \text{Id}$ the identity function.
It is also true when $\phi$ is the log-sum-exp although $\nabla \phi^* \circ \nabla \phi$ is not the identity.
Then the Fenchel duality gap is equal to the Bregman divergence generated by $\phi^*$:
\beq
F_\phi(s, \alpha) = D_{\phi^*}(\alpha || \nabla \phi (s)).
\eeq
Then the duality gap can be written as the average over data points of the $\phi^*$-Bregman divergence between $\alpha_i$ and its next fixed point iterate: $\hat\alpha_i\circ\hat w(\balpha)$:
\beq
g(\hat w(\balpha), \balpha)
= \frac{1}{n}\sum_{i=1}^n D_{\phi^*}(\alpha_i || \hat\alpha_i\circ\hat w(\balpha)).
\label{app:eq:dual_gap}
\eeq

\paragraph{Conjugate dual variable.}
The situation is quite symmetric.
Under the assumption that  $\balpha \vcentcolon= \hat \alpha(\bw)$, one gets:
\beq
g(\bw, \hat \alpha(\bw))
= \lambda \left ( r(\bw) + r^*(\frac{ A \hat \alpha(\bw)}{\lambda n})) - \bw^T\frac{ A \hat \alpha(\bw)}{\lambda n} \right )
= \lambda F_r(\bw, \frac{ A \hat \alpha(\bw)}{\lambda n}),
\eeq
where $F_r$ is the fenchel duality gap of the regularizer.
We can transform it into the Bregman divergence between $\bw$ and its next iterate $\bw' := \nabla r^*(\frac{ A \hat \alpha(\bw)}{\lambda n})) = \hat w \circ \hat \alpha(\bw)$ at the condition that:
\beq
\langle \nabla r \circ \nabla r^*(\bm v) - \bm v, \bw' -\bw \rangle = 0,
\eeq
for all vectors $\bm v$ in the domain of $r^*$ and all vectors $\bw, \bw'$ in the domain of $r$.
Then the duality gap is:
\beq
g(\bw, \hat \alpha(\bw))
= \lambda D_r(\bw ||  \hat w \circ \hat \alpha(\bw) ).
\label{app:eq:primal_gap}
\eeq

Equations~\eqref{app:eq:dual_gap} and~\eqref{app:eq:primal_gap} show that the objective~\eqref{app:primal problem} is also a fixed point problem for the conjugation operations.
The suboptimality can be easily measured as the divergence between a point, either primal or dual and its next iterate.
The divergence is given by the regularizer of the primal problem $r$ or the dual problem $\phi^*$.

\subsection{THEOREMS}
\label{app:theorem}

We state the convergence rates for some variants of SDCA using non-uniform sampling. The proofs follow in the next section.

Denote $h_t := D(\balpha^*) - \E[D(\balpha^{(t)} ]$ the expectation of the dual sub-optimality at step t.
The expectation is over all the possible samplings (the stochastic part of SDCA).
We will bound this value.
One can bound the duality gap $g(\hat w(\balpha),\balpha) := P(\hat w(\balpha)) - D(\balpha)$ at the cost of another constant outside of the exponential (Appendix~\ref{app:bound duality gap}).

\begin{theorem}[Uniform sampling \citep{shalev-shwartz_accelerated_2013-1}]
	\label{app:th:uniform}
	At each step, sample $i$ with uniform probability in $[1,n]$.
	After t iterations, the dual sub-optimality is bounded by:
	\begin{equation}
		h_t \leq (1-\frac{s}{n})^t  h_0,
	\end{equation}
	where $ s = (1+ \frac{R}{n \lambda \strgconvex} )^{-1} $ is the fixed step-size used in the proof.
\end{theorem}

This theorem holds for SDCA with line search as well, since the line search can only be faster than the fixed step size.
None of the following algorithm take the line search into account.
The relative values of the bounds appearing in each theorems may not always reflect the relative
performance of each algorithms.

Intuitively, we want the linear coefficient, here $\frac{s}{n}$, to be as large as possible.
Here $R/\strgconvex$ is the max of the smoothness of the individual losses $\phi_i$.
If the regularizer is smooth enough, then the linear coefficient is related to the condition number $\kappa$ by:
\beq
\frac{n}{s}  = n + R/(\lambda \strgconvex) \approx n + \kappa.
\eeq

The following theorem goes from the maximum radius $R$ to the mean radius $\bar R$.

\begin{theorem}[Importance Sampling \citep{Zhao2015StochasticOptimizationImportance}]
	\label{app:importance}
	At each step, sample $i$ with probability $p_i$ proportional to the individual "condition number":
	\begin{equation}
		p_i \propto 1+R_i/(n \lambda \strgconvex).
	\end{equation}
	After t iterations, the dual sub-optimality is bounded by:
	\begin{equation}
		h_t \leq (1-\frac{\bar s}{n})^t  h_0,
	\end{equation}
	where $\bar s := (1 + \frac{\bar R}{n \lambda \strgconvex})^{-1}$ is the harmonic mean of the step-sizes used in the proof.
\end{theorem}

The harmonic mean is always larger than the minimum step size, so the importance sampling will converge faster than the uniform sampling \textit{at the condition} that we have an accurate estimate of the operator norms $R_i$.
Indeed, if we get the operator norms wrong, then we will sample more often points that are actually easier to classify.
Even if we estimate them right, empirical convergence may be slower with this scheme because of the line search.
This is what happened during the experiments that we ran on CRFs.

Note the similarity with non-uniform sampling in primal methods.
The convergence is improved thanks to larger step sizes, that are proportional to the inverse of some kind of Lipschitz constants.
The convergence rate depends on the arithmetic mean of these Lipschitz constants instead of the max.

We now introduce an adaptive scheme. We reformulate the theorem to make it more compact and comparable with our theorems.

\begin{theorem}[AdaSDCA \citep{csiba2015stochastic} ]
	\label{app:csiba}
	Suppose that the loss functions are \textbf{quadratic} $\phi(z):=\|z\|_2^2$.
	Denote $d_i^t =  \|\beta_i^t - \balpha_i^t \|_{D'}$
	At each step $t$, sample $i$ with probability $p_i^t$ defined by:
	\begin{equation}
		p_i^t \propto d_i^t \sqrt{1+R_i/(n \lambda \strgconvex)},
	\end{equation}

	\begin{equation}
		\theta(\bm d, \bm p) = \frac{\sum_i d_i^2}{\sum_{i | p_i >0} \frac{d_i^2}{p_i} (1+ \frac{R_i}{n \lambda \strgconvex}) },
	\end{equation}
	and
	\begin{equation}
		\tilde \theta_t =  \frac{\E[\theta(\bm d^t,\bm p^t) (P(\bw^t) - D(\balpha^t))] }{ \E[P(\bw^t) - D(\balpha^t) ] }
	\end{equation}
	where the expectation is taken over all the possible trajectories of the algorithm, e.g the sampling of the points.
	Finally define $\tilde \theta = \min_t \tilde \theta_t$.
	After t iterations, the dual sub-optimality is bounded by:
	\begin{equation}
		h_t \leq (1- \tilde \theta )^t  h_0.
	\end{equation}
\end{theorem}

In the theorem above, we have to take the expectation of some variable over all the trajectories of the algorithm.
This is not very clean, but this is unavoidable to get a general convergence result with an adaptive scheme.
Alternatively, one could simply compare the improvement given by one step for each algorithm.

A major limitation of the theorem above is that the loss has to be quadratic.
This theoretical limitation is not a big problem empirically.
It results from a symbolic trick used in the proof : setting the step-size to be proportional to the inverse of the probability.
This is reasonable for importance sampling, because the probability is proportional to the smoothness constant.
Setting the step-size to the inverse of the smoothness is optimal for gradient descent.
This may be less reasonable for other sampling schemes.

Another limitation is that we have to estimate the $n$ distances $d_i^t$ at each step.
In practice we compute $d_i^t$ only for the sampled $i$, and use the latest estimate $d_j^{t'}$ for all the other samples $j$.
Our estimates will become stale as the algorithm unfolds, but there are heuristics to compensate for that phenomenon.
One is to sample from a mixture between a uniform and an adaptive distribution.
Another is to do a batch update of the $d_i$ every once in a while.
These heuristics are unavoidable for adaptive schemes, as we do not want the cost of every update to be $O(n)$.
We do not know how  to analyze the impact of these heuristics.
Empirically, adaptive sampling with this heuristic still accelerates convergence.

Now we are going to introduce two new adaptive sampling scheme.
Both of them rely on the structure of the duality gap:
\begin{equation}
	g(\hat w(\balpha),\balpha) := P(\hat w(\balpha)) - D(\balpha) = \sum_i \phi(-A_i^T \hat w(\balpha)) + \phi^*(\alpha_i) + \langle \hat w(\balpha),  A_i \alpha_i \rangle.
\end{equation}
Each term of the sum above is a Fenchel duality gap between the loss and its convex conjugate.
They are all positive, and somehow represent the sub-optimality of the current model for every training sample.
Intuitively, sampling the most sub-optimal point may yield the best improvement.

\begin{theorem}[Gap sampling]
	\label{app:th:gap}
	At each step $t$, sample $i$ with probability $p_i^t$ proportional to the individual Fenchel duality gap:
	\begin{equation}
		p_i^t \propto g_i^t := \phi(-A_i^T \bw^t) + \phi^*(\alpha_i^t) + \langle \bw^t,  A_i \alpha_i^t \rangle.
	\end{equation}
	Define the non-uniformity of the duality gaps as the ratio between their quadratic mean and their arithmetic mean:
	\begin{equation}
	    \label{app:eq:non-uniformity}
		\chi^2(\bm g) := \frac{\frac{1}{n} \sum_i g_i^2}{  \big ( \frac{1}{n}  \sum_i g_i \big )^2 } \in [1,n].
	\end{equation}
	Take $\chi$ a lower bound on these non-uniformity over all trajectories, for all time steps.
	After t iterations, the dual sub-optimality is bounded by:
	\begin{equation}
		h_t \leq (1-s\frac{\chi^2}{n})^t  h_0.
	\end{equation}
	where $ s = (1+ \frac{R}{n \lambda \strgconvex} )^{-1} $ is the fixed step-size used in the proof.
\end{theorem}

This theorem has the same limitations relative to adaptive scheme that we mentioned for AdaSDCA.

This kind of sampling scheme was studied in the sublinear convergence regime by \citet{osokin2016minding} (Franke-Wolfe) and \citet{perekrestenko17a} (Coordinate Descent).
They could not establish a domination of gap-sampling over uniform sampling.
This is what we prove in the linear regime for SDCA since the non-uniformity $\chi$ belongs to $[1,\sqrt n]$.

The non-uniformity $\chi^2(\bm g)$~\eqref{app:eq:non-uniformity} is worth $1$ if the gaps
are all the same, and $\sqrt n$ if only one gap is non-zero, hence the name.
Gap-sampling will be $n$ times faster than uniform sampling if only one sample $i$ is suboptimal $g_i>0$.
This result is sensible since we will sample only one point, while the uniform algorithm may sample a large number first.
Let us imagine another scenario where all points are already optimal except $k$ of them which have the same gap value.
Then the acceleration coefficient will be $\frac{n}{k}$, which can be a significant acceleration when $k$ is much smaller than $n$.
Finally, consider a scenario where the gaps are evenly distributed $\{a, 2a, ..., n a\}$ for some value $a > 0$.
Note that $\chi^2(\bm g)$ is scale-invariant and does not depend on the specific value $a$.
We can compute $\chi^2(\bm g)$ explicitly here using Faulhaber's formula for the sum of powers of integers:
\begin{equation*}
\chi^2(\bm g) = \frac{\frac{1}{n} \frac{n(n+1)(2n+1)}{6}}{\left(\frac{1}{n}\frac{n(n+1)}{2}\right)^2} = \frac{2}{3} \frac{2n+1}{n+1} \approx 4/3.
\end{equation*}
The acceleration coefficient here is approximately 4/3 compared to uniform sampling.

The duality gaps are often computable, even in the Conditional Random Fields context.
On the other hand, we do not have direct access to the dual variable $\balpha$
and we cannot compute the distance $d_i = \| \beta_i - \alpha_i \|_1$, as it is the $\ell^1$ norm of a vector of exponential size.

Now we want to combine importance sampling with duality gap sampling.
We would like to benefit both from the dependency on $\bar R$ and the acceleration by $\chi$.

\begin{theorem}[Lipschitz-gap sampling]
	\label{app:th:gap+}
	At each step $t$, sample $i$ with probability $p_i^t$ defined by:
	\begin{equation}
		p_i^t \propto g_i^t (1+ R_i/(n \lambda \strgconvex)).
	\end{equation}
	Define $\chi$ as in~\eqref{app:eq:non-uniformity} from Theorem \ref{app:th:gap}.
	Define $\tilde s$ as the quadratic harmonic mean of the step-sizes $s_i := 1/(1+ R_i/(n \lambda \strgconvex))$.
	After t iterations, the dual sub-optimality is bounded by:
	\begin{equation}
		h_t \leq (1-\tilde s \frac{\chi}{n})^t  h_0.
	\end{equation}
\end{theorem}

This theorem makes apparent a trade-off between the advantage gained with the smoothness,
and the advantage gained with the individual gaps.
We lose the square factor on the non-uniformity compared to Theorem \ref{app:th:gap}.
We go from the harmonic mean to the quadratic harmonic mean (generalized norm $-2$) of the step sizes,
which is basically the same as going from the arithmetic mean of the smoothness to the quadratic mean of the smoothness.
Recall that the quadratic mean always lies in between the arithmetic mean and the max.

Our results holds for any smooth loss function, contrary to AdaSDCA.
Our two new strategies complement importance sampling as none of them dominates the other.
Which one is the best depends on the context.
\textit{That is} at the condition that we have access to the $R_i$.
Otherwise gap sampling remains available.

\subsection{PROOFS}
\label{app:proof}

\begin{lemma}[General descent lemma]
\label{app:th:ascent_lemma}
    Apply the SDCA update on the dual variable $\balpha$ to get the new point $\balpha^+$.
    The block $i$ is sampled with probability $p_i$ and updated with a step size $s_i$.
    The expected dual improvement verifies the lower bound:
	\begin{equation}
		n \mathbb E_{\bm p}[D(\balpha^+)] - D(\balpha)
		\geq \underbrace{ \sum_i p_i s_i g_i }_{ \textrm{not the duality gap}}
		+ \frac{\strgconvex}{2} \sum_i p_i s_i
		\bigg ( 1 - s_i  \underbrace{\left (1 + \frac{R_i}{\strgconvex \lambda n} \right )}_{:=c_i} \bigg ) d_i^2
	\end{equation}
    where $\mathbb E_{\bm p}$ denotes the conditional expectation over the choice $i \sim \bm p$
    of block to update, conditioned on the previous state $\balpha$.
\end{lemma}

\begin{proof}[Proof of Lemma \ref{app:th:ascent_lemma}]
This statement is similar to a weighted combination of Equation~(25) from~\citet{shalev-shwartz_accelerated_2013-1}. We provide here the derivation to be self-contained.
    Suppose we sampled the point $i$ and updated the block $\alpha_i$ with step size $s_i$:
    \begin{equation}
        \alpha_i^+ := \alpha_i + s_i \delta_i =  (1- s_i)\alpha_i + s_i \beta_i \, .
    \end{equation}
    The dual improvement is:
    \begin{equation}
    \label{app:eq:dual_improvement}
        n (D(\balpha^+) - D(\balpha))
        = \underbrace{ \lambda n \left( r^* \left (\frac{A \balpha}{\lambda n} \right) - r^* \left( \frac{A \balpha^+}{\lambda n} \right)  \right ) }
        _{\text{data fidelity}}
        + \underbrace{\phi^*(\alpha_i) - \phi^*(\alpha_i^+)}
        _{\text{regularization}} \, .
    \end{equation}

    We first bound the data fidelity term.
    We use the the fact that $r^*$ is 1-smooth with respect to $\|.\|_{D'}$ to upper-bound its variation:
    \begin{equation}
        r^* \left( \frac{A \balpha^+}{\lambda n} \right)
        = r^* \left( \frac{A \balpha}{\lambda n} + s_i \frac{A_i \delta_i}{\lambda n} \right)
        \leq r^* \left( \frac{A \balpha}{\lambda n} \right)
        + s_i \left \langle \nabla  r^* \left( \frac{A \balpha}{\lambda n} \right), \frac{A_i \delta_i}{\lambda n} \right \rangle
        + \frac{s_i^2}{2} \norm{ \frac{A_i \delta_i}{\lambda n} }_{D'}^2
    \end{equation}
    The linear coeficient of this lower boudn is $\hat w( \balpha) = \nabla  r^* \left( \frac{A \balpha}{\lambda n} \right)$.
    The quadratic term can be further upper-bounded:
    \begin{equation}
        \norm{ \frac{A_i \delta_i}{\lambda n} }_{D'}^2
        \leq \frac{1}{(\lambda n )^2} \norm{A_i}_{D \rightarrow D'}^2 \norm{\delta_i}_{D}^2
        = \frac{R_i d_i^2}{(\lambda n )^2} \, ,
    \end{equation}
    by definition of the radius $R_i$ and the residue $d_i := \norm{\beta_i - \alpha_i}_D$.
    So the loss variation is lower bounded by:
    \begin{equation}
    \label{app:eq:bound_loss}
        \lambda n \left( r^* \left (\frac{A \balpha}{\lambda n} \right) - r^* \left( \frac{A \balpha^+}{\lambda n} \right)  \right )
        \geq s_i \left \langle \hat w( \balpha) , A_i (\alpha_i - \beta_i) \right \rangle
        - \frac{s_i^2}{2} \frac{R_i d_i^2}{\lambda n } \, .
    \end{equation}

    Now we bound the regularization term.
    Since $\phi^*$ is $\strgconvex$-strongly convex with respect to $\|.\|_D$,
    \begin{equation}
        \phi^*(\alpha_i^+)
        = \phi^*((1- s_i)\alpha_i + s_i \beta_i )
        \leq (1- s_i) \phi^*(\alpha_i) + s_i \phi^*(\beta_i)
        - s_i (1 - s_i ) \frac{\strgconvex}{2} d_i^2 \, .
    \end{equation}
    The regularization variation can be lower bounded by:
    \begin{equation}
    \label{app:eq:bound_regularization}
        \phi^*(\alpha_i) - \phi^*(\alpha_i^+)
        \geq s_i \left ( \phi^*(\alpha_i) -  \phi^*(\beta_i) \right )
        + s_i (1 - s_i ) \frac{\strgconvex}{2} d_i^2 \, .
    \end{equation}

    Plugging the bounds~\eqref{app:eq:bound_loss} and~\eqref{app:eq:bound_regularization} into Equation~\eqref{app:eq:dual_improvement}, we get:
    \begin{equation}
        n (D(\balpha^+) - D(\balpha))
        \geq s_i \left (\phi^*(\alpha_i)
        + \left \langle \hat w( \balpha) , A_i (\alpha_i - \beta_i) \right \rangle
         -  \phi^*(\beta_i) \right )
         + \frac{s_i}{2} \left ( (1 - s_i ) \strgconvex - s_i \frac{R_i}{\lambda n } \right ) d_i^2 \, .
    \end{equation}
    Recall that $\beta_i := \nabla \phi (-A_i^T \hat w (\balpha))$. Thus,
    \begin{equation}
        \left \langle - A_i^T \hat w( \balpha) , \beta_i \right \rangle - \phi^*(\beta_i)
        = \phi(- A_i^T \hat w( \balpha))
    \end{equation}
    by definition of the convex conjugate $\phi^*$.
    To sum up, at iteration t, if we sample the block i, and update it with step size $s_i$,
    we can lower bound the resulting dual improvement with:
    \begin{equation}
        \label{app:one point descent}
        n (D(\balpha^+) - D(\balpha))
        \geq s_i \underbrace{ \big [ \phi(-A_i^T \hat w(\balpha)) + \phi^*(\alpha_i) + \hat w(\balpha)^T A_i \alpha_i \big ]
        }_{ \textrm{Fenchel gap} =: g_i}
        + \frac{s_i \strgconvex}{2} \left ( 1 - s_i\left (1 + \frac{R_i}{\strgconvex \lambda n } \right )  \right )  d_i^2 \, .
    \end{equation}
    To conclude the proof, take a weighted average of the inequalities~\eqref{app:one point descent} with the weights $p_i$.
\end{proof}

In the following we note the duality gap:
\begin{equation}
    \bar g := \frac{1}{n} \sum_i g_i = P(\hat w(\balpha)) - D(\balpha) \, .
\end{equation}

\begin{proof}[Proof of Theorem \ref{app:th:uniform}]
    In the original proof of \citet{shalev-shwartz_accelerated_2013-1}, we set $p_i=1/n$ and $s_i = s = (1+ \frac{R}{n \lambda \strgconvex} )^{-1}  \leq 1/c_i$. This step size guarantees that the right hand term is positive, leaving us with the inequality:
    \begin{equation}
        \mathbb E_{p^t}[D(\balpha^{t+1}) - D(\balpha^t)]
        \geq \frac{s}{n} \bar g^t.
        \label{app:eq:uniform ascent lemma}
    \end{equation}
    Now observe that $\mathbb E_p[D(\balpha^+) - D(\balpha)] = - \mathbb E_p[h_{t+1}] + h_t$ and $\bar g^t = (P(\bw^t) - D(\balpha^t)) \geq h_t$.
    Moving the sub-optimality at time $t$ on the right gives:
    \beq
        \mathbb E_p[h_{t+1}] \leq (1- \frac{s}{n} ) h_t.
    \eeq
    This inequality is conditional on all the random sampling until time $t$.
    Let us take the expectation of this inequality with respect to all this past randomness.
    We get a recursive upper bound on the expected dual sub-optimality:
    \beq
        \mathbb E [h_{t+1}] \leq (1- \frac{s}{n} ) \mathbb E [h_t] \leq (1- \frac{s}{n} )^t h_0.
    \eeq
    This is the final convergence result with the linear constant $s/n = (n+R/(\lambda \strgconvex))^{-1}$.
\end{proof}

In the proof above, we lower bound the dual improvement by the duality gap, then we use this to get the linear convergence rate.
All the proofs follow the same reasoning, and the last few steps are always the same so we will skip them.

\begin{proof}[Proof of Theorem \ref{app:importance}]
    Inject $p_i=c_i/\sum_j c_j$ and $s_i = 1/c_i$. The right hand term is zero thanks to the step-size, hence the lower bound:
    \begin{equation}
        \mathbb E_p[D(\balpha^+) - D(\balpha)]
        \geq \frac{\bar g}{\sum_i c_i} .
    \end{equation}
    We get the linear rate   $\frac{1}{\sum_i c_i}$ which is also the harmonic mean of the step-sizes divided by $n$.
\end{proof}

\begin{proof}[Sketch of Proof of Theorem \ref{app:csiba}]
To make the duality gap appear in this formula for arbitrary probability $p$, \citet{csiba2015stochastic} use $p_i s_i = \theta$ constant, whenever $g_i > 0$.
If the individual duality gap is null $g_i=0$, then they set $p_i=s_i=0$.
\begin{equation}
    \mathbb E_p[D(\balpha^+) - D(\balpha)] - \theta \bar g
    \geq \theta \frac{\strgconvex}{2 n} \sum_i  d_i^2\bigg ( 1 -  \frac{\theta}{p_i} \big ( 1 - \frac{R_i^2}{\strgconvex \lambda n} \big ) \bigg )
\end{equation}

The negative consequence of that strategy is that they have to enforce $s_i \in [0,1]$ by setting
$\theta < \min_i p_i$ where the minimum is taken over the sub-optimal i's (i.e. $p_i>0$).
This a terrible constraint on the step size, as we cannot be too non-uniform without taking very small steps.
It effectively reduces the linear convergence constant $\theta /n$.

Finally, they want to maximize $\theta$ while keeping the right hand side positive.
This is a hard problem on $\theta$ and $p$.
When the loss is the quadratic loss, they can remove the condition that the step-size should be smaller than 1.
Then they solve the optimization problem to get the sampling scheme $p_i \propto d_i \sqrt c_i$.
\end{proof}

\begin{proof}[Proof of Theorem \ref{app:th:gap}]
    We use the same step-size as in the original proof:
    \begin{equation}
        s_i = s = \frac{n}{n+R/(\lambda \strgconvex)}.
    \end{equation}
    We have the guarantee that the right hand term is positive. The lemma simplifies to:
    \begin{equation}
        n \mathbb E_p[D(\balpha^+) - D(\balpha)]
        \geq \frac{s}{n} \sum_i p_i g_i.
    \end{equation}
    We inject $p_i= \frac{g_i}{n\bar g}$ into this lower bound:
    \begin{equation}
        \mathbb E_p[D(\balpha^+) - D(\balpha)]
        \geq \frac{s}{n} \frac{\sum_i g_i^2}{\sum_j g_j}
        = \frac{s}{n} \chi^2(\bm g) \bar g\, ,
    \end{equation}
    where we introduced the non-uniformity of the duality gaps vector defined in Equation~\eqref{app:eq:non-uniformity}.
    To get a simpler expression for a global convergence bound, let us define $\chi$ to be a lower bound on $\chi(\bm g)$ over all the possible unfolding of SDCA and for every steps.
    Now we can write the descent lemma in the same form as in the original proof, but a with new constant:
    \begin{equation}
        \mathbb E_p[D(\balpha^+) - D(\balpha)]
        \geq \frac{s}{n} \chi^2 \bar g\, .
    \end{equation}
\end{proof}

\begin{proof}[Proof of Theorem \ref{app:th:gap+}]
	We set $p_i \propto g_i c_i$ where $c_i = 1+ R/(n \lambda \strgconvex)$.
	\begin{equation}
		n \mathbb E_p[D(\balpha^+) - D(\balpha)]
		\geq \frac{ \sum_i s_i g_i^2 c_i}{\sum_i g_i c_i}
		+ \frac{ \frac{\strgconvex}{2}  \sum_i s_i  g_i c_i d_i^2 \
		\big ( 1 - s_i c_i \big) }{\sum_i g_i c_i}
	\end{equation}
	Similarly to the proof of importance sampling, we now set $s_i= 1/c_i \leq 1$ instead of $s_i=s=1/\max_i c_i$.
	This nullifies the right hand term.
	We can take longer steps if the individual Lipschitz constants are high.
	\begin{equation}
		n \mathbb E_p[D(\balpha^+) - D(\balpha)]
		\geq \frac{ \sum_i g_i^2 }{\sum_i g_i c_i} = \frac{\langle \bm g, \bm g \rangle}{\langle \bm c, \bm g \rangle}
	\end{equation}
	We apply the Cauchy-Schwartz inequality : $\langle \bm c, \bm g \rangle \leq \|\bm c\|_2 \|\bm g \|_2$.
	 \begin{equation}
		n \mathbb E_p[D(\balpha^+) - D(\balpha)]
		\geq \frac{ \|g\|_2 }{\|c\|_2}
		=  \frac{\chi(g)}{\QM(c)} \bar g \, ,
	\end{equation}
	where $\QM$ denotes the quadratic mean. Finally we divide both sides by n to complete the proof:
	\begin{equation}
		\mathbb E_p[D(\balpha^+) - D(\balpha)]
		\geq \frac{ \chi(g) }{n \QM(c)} \bar g \, .
	\end{equation}
\end{proof}

\end{document}